\documentclass{article}

\usepackage[main, final]{neurips_2026}

\usepackage[utf8]{inputenc} 
\usepackage[T1]{fontenc}    
\usepackage{hyperref}       
\usepackage{url}            
\usepackage{booktabs}       
\usepackage{amsfonts}       
\usepackage{nicefrac}       
\usepackage{microtype}      
\usepackage{xcolor}         
\usepackage{placeins}
\usepackage{pgfplots}
\usepgfplotslibrary{fillbetween}
\pgfplotsset{compat=1.18}

\definecolor{DeepNavy}{HTML}{045A8D}
\definecolor{DeepRed}{HTML}{A50F15}
\definecolor{BurntOrange}{HTML}{D94801}
\definecolor{DeepGreen}{HTML}{006D2C}
\definecolor{DeepPurple}{HTML}{4A1486}

\pgfplotsset{
    modern paper/.style={
        grid=major,
        grid style={dashed, gray!30},
        axis line style={gray!80, thick},
        tick align=outside,
        tick style={gray!80, thick},
        legend style={draw=none, fill=white, fill opacity=0.8, text opacity=1, nodes={scale=0.8, transform shape}},
        font=\small
    }
}

\usepackage{amsmath} 
\hypersetup{
colorlinks=true
}

\usepackage[utf8]{inputenc}
\usepackage[T1]{fontenc}
\usepackage{amsmath, amssymb, amsthm, mathtools}
\usepackage{graphicx}
\usepackage{pgfplots}
\pgfplotsset{compat=1.18}
\usepackage{booktabs}
\usepackage{longtable}
\usepackage{multirow}
\usepackage{array}
\usepackage{xcolor}
\usepackage{enumitem}
\usepackage{hyperref}
\usepackage{cleveref}
\usepackage{algorithm}
\usepackage{algpseudocode}

\theoremstyle{plain}
\newtheorem{theorem}{Theorem}
\newtheorem{proposition}{Proposition}

\newtheorem{corollary}{Corollary}

\theoremstyle{definition}
\newtheorem{assumption}{Assumption}

\theoremstyle{plain}

\newtheorem*{customthm}{Theorem}

\DeclareMathOperator*{\argmax}{arg\,max}
\DeclareMathOperator*{\argmin}{arg\,min}
\DeclareMathOperator{\E}{\mathbb{E}}

\newcommand{\hA}{\hat{A}}
\newcommand{\hJ}{\hat{J}}
\newcommand{\Wphi}{\mathcal{W}_{\phi}}
\newcommand{\pitheta}{\pi_{\theta}}
\newcommand{\pithetap}{\pi_{\theta^{+}(\phi)}}

\newcommand{\lossvwi}{\mathcal{L}_{\mathrm{VWI}}}

\newcommand{\lossmse}{\mathcal{L}_{\mathrm{MSE}}}
\newcommand{\PUR}{\ensuremath{\mathrm{PUR}}}

\newcommand{\para}[1]{\vspace{0.25em}\noindent\textbf{#1}}

\title{DRILL: Training World Models to Improve Policies, Not Predict Pixels}

\author{%
  \textbf{Zhaolu Kang}$^{1,2}$, \\\textbf{Tailong Luo}, \textbf{Lei Wei}$^{2}$, \textbf{Guangyuan Dong},
  \textbf{Sihan Liu}$^{2}$,\textbf{Yingjie He}$^{2}$,
\textbf{Siheng Wang}, \\\textbf{Shuaibo Li},\textbf{Rongchao Zhang},  \textbf{Zhen Tian}, 
  \textbf{Richeng Xuan}$^{1,\dagger}$ ,\textbf{Zhichao Hu}$^{1,\dagger}$ \\
  $^{1}$Hunyuan, Tencent \quad $^{2}$Peking University  \\
  \texttt{zlkang25@stu.pku.edu.cn}
}

\begin{document}

\maketitle

\renewcommand{\thefootnote}{\fnsymbol{footnote}}
\footnotetext{$^{\dagger}$Corresponding Author.}
\renewcommand{\thefootnote}{\arabic{footnote}}

%
\begin{abstract}

World models are increasingly used to reduce environment interaction when post-training vision-language-action (VLA) policies with reinforcement learning. Yet they are typically trained to predict pixels one step ahead, even though their role here is not merely to render visually faithful videos but to produce learning signals from imagined rollouts that improve the policy. We show that this mismatch is substantial: lower pixel MSE can produce worse policy gradients, and a world model trained for policy utility can yield better downstream policies despite 17\% higher pixel MSE on the next frame. 
We propose the Downstream-Return Imagination Learning Loop (DRILL), a bilevel framework that updates the world model to maximize the policy's return after an inner GRPO update on imagined rollouts. DRILL has two instantiations. DRILL-VWI is a closed-form surrogate, requiring no additional simulator rollouts, that reweights prediction errors by policy visitation, the group-relative GRPO advantage, and model confidence. DRILL-IMG is a full meta-gradient method that differentiates through the inner update using truncated Hessian-vector products. We show that DRILL-VWI recovers the local one-step component of the DRILL-IMG meta-gradient under score-compatibility and small-residual assumptions. Across five manipulation simulators and two VLA backbones, DRILL-IMG raises mean success rate by 13.1 points over WMPO and 7.9 points over RLVR-World, and complements RLVR pretraining rather than competing with it. More broadly, we believe video world models for VLA reinforcement learning can be productively trained by the policy gradients they produce, not only by pixel-level prediction.
\end{abstract}


\section{Introduction}
\label{sec:intro}

World models are becoming a practical substitute for robot or simulator
interaction in reinforcement-learning post-training of
vision-language-action (VLA) policies.  In this setting, however, their
role is not merely to predict the next video frame: they generate the
imagined rollouts from which the policy learns.  This distinction
creates an objective mismatch.  A model optimized for one-step pixel
prediction is not necessarily the model that induces the most useful
policy-gradient update.

The reason is that pixel fidelity and policy utility weight errors in
different places.  Pixel MSE penalizes prediction errors uniformly over
pixels and time, including backgrounds, robot geometry, and visually
salient but task-irrelevant regions.  A policy-gradient update instead
emphasizes the states the current policy visits and the actions whose
advantages are large.  A model can therefore be an excellent next-frame
predictor while producing weak learning signals, or a worse predictor on
average while better supporting the transitions that determine policy
improvement.  The relevant question is not only whether the imagined
video looks right, but whether learning inside it moves the policy in
the right direction.

This mismatch matters because recent VLA post-training systems replace
expensive robot or simulator interaction with imagined rollouts.  At
foundation-model scale, reinforcement-learning post-training can require
millions of environment steps, so recent methods train a video world
model from offline demonstrations and apply Group Relative Policy
Optimization (GRPO)~\citep{grpo2024} inside the model
\citep{wmpo2025, vlaw2025, wovr2025, scvla2025}.  Yet the world model
is still typically optimized as a one-step video predictor.  Diagnostic
methods such as Policy-Aware WM~\citep{policyaware2026} and
ESNR~\citep{esnr2026} show that pixel fidelity does not reliably track
policy competence, but do not provide a training objective.  RLVR-World~\citep{rlvrworld2025}
post-trains the world model with prediction correctness as the reward,
improving it as a predictor rather than directly optimizing the policy
update induced by imagined rollouts.

We propose the \textbf{D}ownstream-\textbf{R}eturn
\textbf{I}magination \textbf{L}earning \textbf{L}oop (DRILL), a
bilevel framework that trains video world models by the return of the
policy obtained after learning in them.  The inner loop performs a GRPO
update on imagined rollouts generated by the world model; the outer loop
updates the world model so that this post-update policy achieves higher
return in the target simulator.  The objective is therefore not pixel
accuracy under the model's own rollouts, but downstream return in the
environment that defines task success.

DRILL has two operating points.  DRILL-VWI (Visitation-Weighted
Imagination) is a lightweight signed surrogate that replaces uniform
pixel MSE with weights from policy visitation, group-relative advantage,
and model confidence, requiring no additional simulator rollouts over a
WMPO-style pipeline.  DRILL-IMG (Implicit Meta-Gradient) estimates the
downstream-return meta-gradient more directly by differentiating through
the inner GRPO update with truncated Hessian-vector products and a small
simulator-rollout batch.  Thus DRILL-VWI is the budget-conscious drop-in
update, while DRILL-IMG is the more faithful estimator when limited
simulator interaction and BPTT memory are available.

The two variants are connected theoretically: the core DRILL-VWI
surrogate is the first-order Taylor approximation of the DRILL-IMG
meta-gradient under standard local assumptions, and DRILL-IMG admits a
non-asymptotic convergence rate with truncated backpropagation through
the inner update.  This explains when a local visitation- and
advantage-weighted prediction loss can approximate the more expensive
meta-gradient, and when longer-horizon trajectory effects make DRILL-IMG
worth the additional cost.

Across five public manipulation simulators
\citep{libero2023, calvin2022, robocasa2024} and two foundation VLA
backbones~\citep{pi0_2024, openvla2024}, DRILL-IMG raises mean policy
success rate by 13.1 points over WMPO and 7.9 points over RLVR-World,
while its world model has 17\% higher one-step pixel MSE on the same
training data.  The gains are consistent across seeds, increase with
rollout horizon and out-of-distribution shift, and remain complementary
to prediction-correctness pretraining: applying DRILL-IMG to an
RLVR-World-pretrained world model yields a further 3.2-point improvement
on average.  We also introduce the Policy-Utility Residual (\PUR{}), a
method-agnostic diagnostic measuring the angular gap between the
pixel-MSE gradient and the policy-improvement direction; across the
objectives we evaluate, lower \PUR{} tracks higher downstream success.

Our contributions are as follows:
\begin{itemize}[leftmargin=*, noitemsep, topsep=0pt]
  \item We identify and formalize an objective mismatch in VLA
        world-model post-training: world models are trained for
        one-step pixel prediction, but used to generate policy-learning
        signals for imagined-rollout reinforcement learning.

  \item We introduce DRILL, a bilevel framework that trains the world
        model for the downstream return of the policy obtained after an
        inner GRPO update, and instantiate it as both a no-extra-rollout
        surrogate (DRILL-VWI) and a simulator-grounded meta-gradient
        method (DRILL-IMG).

  \item We show that DRILL-VWI is a first-order approximation to the
        DRILL-IMG meta-gradient, provide a non-asymptotic convergence
        guarantee for DRILL-IMG under truncated BPTT, and validate the
        framework across five simulators and two VLA backbones with
        higher downstream success despite higher pixel MSE.

  \item We introduce \PUR{} as a method-agnostic diagnostic of the
        angular gap between pixel prediction and policy improvement,
        and show that it tracks downstream success across objectives.
\end{itemize}


\section{Related Work}
\label{sec:related}

\para{Video world models for VLA post-training.}
Recent VLA post-training systems use video world models to replace
expensive robot or simulator interaction with imagined rollouts.
WMPO~\citep{wmpo2025}, VLAW~\citep{vlaw2025}, WoVR~\citep{wovr2025},
SC-VLA~\citep{scvla2025}, Ctrl-World~\citep{ctrlworld2025},
World4RL~\citep{world4rl2026}, and
World-Gymnast~\citep{worldgymnast2026} pretrain the world model from
demonstrations and then update the policy inside the learned dynamics.
These methods differ in rollout construction, filtering, and policy
optimization details, but the world model is still primarily trained as
a next-frame predictor.  Policy-Aware WM~\citep{policyaware2026} and
ESNR~\citep{esnr2026} diagnose the gap between prediction quality and
policy competence, while RLVR-World~\citep{rlvrworld2025} and
WorldCompass~\citep{worldcompass2026} improve the predictor with
prediction-correctness rewards.  DRILL instead optimizes the world
model by the downstream return of the policy after learning from its
imagined rollouts; it is therefore complementary to RLVR-style
pretraining rather than a replacement for it.

\para{Policy-aware world-model learning.}
The idea that a model should be useful for control, not only accurate
under a prediction loss, appears in value-aware model learning
\citep{farahmand2017value, grimm2020value} and in neural model-based
RL systems such as MuZero~\citep{muzero2020}, IRIS~\citep{iris2022},
and DreamerV3~\citep{dreamerv3}.  DRILL brings this perspective to VLA
post-training with imagined GRPO rollouts.  Classical value-aware
objectives rely on a learned $V_\psi(o)$ or bootstrapped targets;
DRILL uses the group-relative advantage produced online by GRPO as a
per-batch surrogate for downstream utility.  Proposition~\ref{prop:v1-is-v2-first-order}
shows that the signed advantage weighting in DRILL-VWI is recovered as
a first-order approximation to the DRILL-IMG meta-gradient under a
local score-compatibility condition, and Section~\ref{sec:exp-ablations}
compares this approximation with a direct value-aware adaptation.

\para{Meta-gradient and implicit differentiation.}
DRILL-IMG is also related to methods that differentiate through an
inner optimization step, including truncated MAML~\citep{maml2017},
iMAML~\citep{imaml2019}, and implicit-gradient methods for RL
\citep{liu2021implicit}.  Unlike iMAML-style approaches, DRILL-IMG
uses a short unroll of the policy update rather than solving the inner
problem to convergence.  Our analysis keeps the resulting truncation
bias explicit: Theorem~\ref{thm:v2-convergence} expresses the BPTT bias
as a function of the truncation length $K$.  The practical distinction
is scale and target: DRILL-IMG applies meta-gradients to
foundation-scale VLAs paired with $>$1B-parameter video world models,
using a small simulator-rollout batch to train the world model for the
post-update return of the policy.


\section{The DRILL Framework}
\label{sec:framework}

DRILL couples two updates.  The inner update uses the current world
model as the environment for one GRPO step in imagination.  The outer
update evaluates the policy produced by that step in the target
environment and assigns credit back to the world model.  DRILL-VWI and
DRILL-IMG differ only in how much of this outer gradient they estimate.
Figure~\ref{fig:drill-overview} gives the full training loop and marks
where the two estimators enter.

\begin{figure}[t]
\centering
\includegraphics[width=0.96\linewidth]{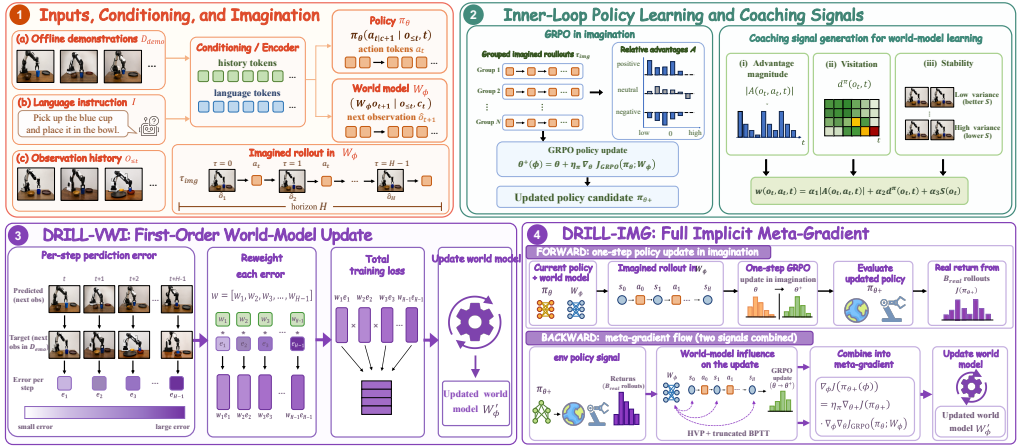}
\caption{Overview of DRILL.  The policy first learns inside the current
world model through imagined GRPO.  The resulting policy-improvement
signal is then fed back to the world model: DRILL-VWI uses the same
imagined rollouts to form a signed, visitation-weighted local surrogate,
whereas DRILL-IMG differentiates through the inner update and grounds the
outer gradient with a small simulator batch.}
\label{fig:drill-overview}
\end{figure}

\subsection{Problem setup}
\label{sec:framework-setup}

We consider a VLA policy $\pitheta(a_{t:t+H_a}\mid o_{\le t}, \ell)$,
a video world model $\Wphi(o_{t+1}\mid o_{\le t}, a_t)$ with predicted
next-frame mean $\hat o_{t+1}(\phi)$, and a frozen vision-language
reward model $r_\psi(o_t,\ell)\in[0,1]$.  Training uses offline
demonstrations $\mathcal{D}_{\mathrm{demo}}$ and imagined rollouts of
horizon $H$ produced by $\pitheta$ under $\Wphi$.  We write
$d_{\rm env}^{\pi_\theta}(o_t)$ for the on-policy state visitation in
the target environment, which in our experiments is the high-fidelity
simulator that defines task return (Section~\ref{sec:exp-setup}), and
$\tilde d_{\phi}^{\pi_\theta}(o_t)$ for the corresponding imagined-rollout
visitation induced by $\pitheta$ inside $\Wphi$.

The standard WMPO recipe~\citep{wmpo2025, vlaw2025} alternates between
fitting the world model to demonstration pixels and improving the
policy inside the model:
\begin{align}
  \phi &\;\leftarrow\; \argmin_{\phi}\;
      \E_{(o_{\le t}, a_t, o_{t+1}) \sim \mathcal{D}_{\mathrm{demo}}}
      \bigl\lVert \hat o_{t+1}(\phi) - o_{t+1}\bigr\rVert_2^2,
    \label{eq:wmpo-mse}\\
  \theta &\;\leftarrow\; \theta \;+\; \eta_\theta\,
      \nabla_\theta \hJ_{\mathrm{GRPO}}\bigl(\pitheta;\Wphi, r_\psi\bigr).
    \label{eq:wmpo-grpo}
\end{align}
The first update is optimized for pixel fidelity, while the second uses
the model through the policy-gradient signal it induces.  DRILL
replaces the pixel-fitting update with one aligned to this downstream
role.

\subsection{The DRILL objective}
\label{sec:framework-bilevel}

Let
\begin{equation}
  \theta^{+}(\phi)
  \;=\; \theta + \eta_\theta\, \nabla_\theta
        \hJ_{\mathrm{GRPO}}(\pitheta; \Wphi)
  \label{eq:inner-update}
\end{equation}
be the policy produced by one imagined GRPO step.  DRILL trains
$\Wphi$ so that this updated policy obtains high return in the target
environment:
\begin{equation}
  \phi^{*}
  \;=\; \argmax_{\phi}\; J\!\bigl(\pithetap\bigr),
  \qquad
  J(\pi) \;=\; \E_{\tau \sim P_{\mathrm{env}}(\pi)}
    \Bigl[\textstyle\sum_t r_t\Bigr].
  \label{eq:drill-objective}
\end{equation}
The expectation in $J$ is taken under the trajectory distribution of
the target environment, not under $\Wphi$: the world model is not asked
to evaluate itself.  Optimizing~\eqref{eq:drill-objective}
requires the meta-gradient $\nabla_\phi J(\pithetap)$ through the
inner update~\eqref{eq:inner-update}.  This formulation changes what the
world model is optimized for: rather than fitting all demonstrated pixels
equally, DRILL rewards model changes that make the imagined policy update
improve target-environment return.  The useful signal is therefore not
frame fidelity alone, but whether the model induces the right
policy-gradient direction on states the current policy actually visits.

DRILL uses two operating points.  DRILL-VWI is a local signed surrogate
requiring no additional simulator rollouts, while DRILL-IMG estimates the
downstream-return meta-gradient more directly with truncated
differentiation through the inner update and a small simulator batch for
the outer return gradient.

\subsection{DRILL-VWI: Visitation-Weighted Imagination}
\label{sec:framework-v1}

DRILL-VWI replaces the pixel-MSE loss~\eqref{eq:wmpo-mse} with a
policy-aware signed local surrogate,
\begin{equation}
  \lossvwi(\phi)
  \;=\;
  \E_{(o_t, a_t, o_{t+1}) \sim \mathcal{D}_{\mathrm{demo}}}
  \Bigl[
    w(o_t, a_t)
    \bigl\lVert \hat o_{t+1}(\phi) - o_{t+1}\bigr\rVert_2^2
  \Bigr],
  \label{eq:vwi-loss}
\end{equation}
with
\begin{equation}
  w(o_t, a_t)
  \;=\;
  \widehat{\tilde d}_{\phi}^{\pi_\theta}(o_t)
  \cdot \widetilde A_\alpha(o_t,a_t) \cdot s(o_t),
  \qquad
  \widetilde A_\alpha(o_t,a_t)
  = \operatorname{sgn}(\hA(o_t,a_t))\,
    \bigl\lvert\hA(o_t,a_t)\bigr\rvert^\alpha .
  \label{eq:vwi-weight}
\end{equation}
The three factors are quantities already produced by the imagined GRPO
loop: $\widehat{\tilde d}_{\phi}^{\pi_\theta}(o_t)$ estimates
model-rollout visitation, $\hA(o_t,a_t)$ is the group-relative GRPO
advantage, and $s(o_t)>0$ downweights states on
which the world model is uncertain.  The weight is stopped and clipped
in implementation, so~\eqref{eq:vwi-loss} is a local update direction,
not a globally normalized likelihood objective.  Proposition~\ref{prop:v1-is-v2-first-order}
formalizes its first-order connection
to the DRILL-IMG meta-gradient.  Algorithm~\ref{alg:vwi} in
Appendix~\ref{app:vwi-algorithm} gives the full training loop,
including clipping, normalization, and visitation estimation.  Since all
factors come from imagined rollouts, DRILL-VWI adds no environment
interaction over WMPO.

\subsection{DRILL-IMG: Implicit Meta-Gradient}
\label{sec:framework-v2}

DRILL-IMG estimates the meta-gradient of~\eqref{eq:drill-objective} by
differentiating through the inner update~\eqref{eq:inner-update}:
\begin{equation}
  \nabla_\phi J\!\bigl(\pithetap\bigr)
  \;=\;
  \eta_\theta\,
  \nabla_{\theta^+} J\!\bigl(\pi_{\theta^+}\bigr)\,
  \nabla_\phi\nabla_\theta \hJ_{\mathrm{GRPO}}(\pitheta;\Wphi).
  \label{eq:meta-gradient}
\end{equation}
The first factor is a policy-gradient estimate in the target simulator,
computed from a small batch of simulator rollouts.  The second factor
is a mixed derivative through the imagined GRPO update, computed as a
Hessian-vector product with $K$-step truncated BPTT.  In the main
experiments we use $K{=}2$ and $B_{\mathrm{real}}{=}30$ simulator
rollouts per outer step.  DRILL-IMG is therefore a more faithful
estimator of the DRILL objective than DRILL-VWI, at the cost of
additional simulator interaction and BPTT memory.  Algorithm~\ref{alg:img}
in Appendix~\ref{app:v2-algorithm} gives the full pseudocode,
checkpointing details, and numerical-stability controls.


\section{Theoretical Analysis}
\label{sec:analysis}

The analysis has three roles.  First, it explains when the cheap
DRILL-VWI update is a first-order surrogate for the DRILL-IMG
meta-gradient.  Second, it states the convergence guarantee for
DRILL-IMG under truncated BPTT.  Third, it defines a diagnostic,
\PUR{}, for measuring whether a world-model objective is aligned with
policy improvement.

\subsection{First-order meta-gradient and the DRILL-VWI surrogate}
\label{sec:analysis-prop}

The inner update~\eqref{eq:inner-update} is a single gradient step, so
the outer objective $J(\pithetap)$ can be expanded in the inner step
size $\eta_\theta$.  The proposition below keeps only the main statement
needed in the paper body; the full proof and residual definitions are
given in Appendix~\ref{app:prop-v1-v2}.

\begin{proposition}[DRILL-VWI as a first-order local surrogate]
\label{prop:v1-is-v2-first-order}
Suppose (A1) $J$ is $C^2$ in $\theta$; (A2) $\Wphi$ has a Gaussian
observation head with mean $\hat o_{t+1}(\phi)$ and fixed scale
$\sigma_W$; and (A3) the imagined-rollout visitation marginal
$\tilde d_{\phi}^{\pi_\theta}$ is locally stationary over one inner
step.  Treat the GRPO advantages as
stop-gradient quantities when forming
$\nabla_\phi\nabla_\theta \hJ_{\mathrm{GRPO}}$.  Let
$\lossvwi^{\mathrm{core}}$ denote the idealized VWI surrogate with exact
imagined-rollout visitation, $\alpha=1$, $s(o_t)=1$, and unclipped signed
weights.
If the local policy-score contraction is compatible with the signed
advantage direction up to residual $\epsilon_{\mathrm{comp}}$, the
omitted multi-step trajectory-score terms have residual $R_H$, and the
demo-anchor transfer has residual $R_{\mathrm{demo}}$, then for some
alignment scalar $\kappa>0$,
\begin{equation}
  \left\|
  \nabla_\phi J\!\bigl(\pithetap\bigr)
  +
  \frac{\kappa\eta_\theta}{2\sigma_W^2}
  \nabla_\phi \lossvwi^{\mathrm{core}}(\phi)
  \right\|
  \le
  \eta_\theta M_W \epsilon_{\mathrm{comp}}
  +
  \eta_\theta R_H
  +
  \eta_\theta R_{\mathrm{demo}}
  +
  \mathcal{O}(\eta_\theta^2).
  \label{eq:vwi-compat-bound}
\end{equation}
\end{proposition}

The result should not be read as an exact identity between DRILL-VWI
and DRILL-IMG.  It says that the signed VWI update recovers the local
one-step component of the DRILL-IMG meta-gradient when the
policy-score contraction is approximately scalar on the
advantage-weighted support and when trajectory-locality and demo-anchor
residuals are small.  Intuitively, positive-advantage transitions are
sharpened, while negative-advantage transitions may be deliberately
deemphasized rather than fit more accurately as pixels.

The residuals describe when this local approximation is reliable.
$\epsilon_{\mathrm{comp}}$ measures whether the world-model score and the
policy-gradient contraction agree on which transition changes matter;
$R_H$ grows when multi-step consequences dominate local next-frame error;
and $R_{\mathrm{demo}}$ captures mismatch between demonstration anchors
and on-policy imagined transitions.  Thus short-horizon or well-covered
tasks favor DRILL-VWI, while large trajectory or demo-anchor residuals
motivate the more expensive DRILL-IMG estimator.

\subsection{Convergence rate under truncated BPTT}
\label{sec:analysis-conv}

DRILL-IMG replaces the exact mixed derivative in
\eqref{eq:meta-gradient} by a $K$-step truncated BPTT estimate.  The
following theorem analyzes one fixed-policy outer subproblem.  Fix
current policy parameters $\bar\theta$ and define
\[
  J_{\bar\theta}(\phi)
  :=
  J\!\left(
  \pi_{\bar\theta+\eta_\theta
  \nabla_\theta \hJ_{\mathrm{GRPO}}(\pi_{\bar\theta};\Wphi)}
  \right).
\]
This is a per-outer-subproblem guarantee, not a proof of global
convergence for the jointly evolving $(\theta_t,\phi_t)$ process.

\begin{theorem}[DRILL-IMG convergence rate under truncated BPTT]
\label{thm:v2-convergence}
Suppose (A4) the $K$-truncated meta-gradient estimator has conditional
variance at most $\sigma_g^2$; (A5) $J_{\bar\theta}$ is $L$-smooth in
$\phi$ and bounded above with
$J^{*}:=\sup_\phi J_{\bar\theta}(\phi)<\infty$; and (A6) the
$K$-step BPTT truncation bias satisfies
$\lVert\mathrm{bias}_K\rVert \le B\rho^{K}$ for some $\rho\in(0,1)$.
Let $J_0:=J_{\bar\theta}(\phi_0)$.  With step size
\[
  \eta_\phi =
  \min\!\left(
    \frac{1}{4L},
    \sqrt{
      \frac{2(J^{*}-J_0)}{L T \sigma_g^2}
    }
  \right),
\]
the $T$-step stochastic ascent iterates on $J_{\bar\theta}$ satisfy
\begin{equation}
  \min_{0\le t<T}\;
  \E\bigl\lVert \nabla_\phi J_{\bar\theta}(\phi_t)\bigr\rVert^2
  \;\le\;
  \mathcal{O}\!\Bigl(\tfrac{L(J^{*}-J_0)}{T}\Bigr)
  \;+\;
  \mathcal{O}\!\Bigl(\sqrt{\tfrac{L(J^{*}-J_0)\sigma_g^2}{T}}\Bigr)
  \;+\;
  \mathcal{O}\!\bigl(B^2 \rho^{2K}\bigr).
  \label{eq:convergence-rate}
\end{equation}
\end{theorem}

The three terms correspond to optimization error, stochastic-gradient
variance, and truncation bias.  The guarantee is qualitative at our
scale: it shows convergence to a neighborhood of a stationary point and
that the squared-gradient contribution from BPTT truncation decays
geometrically in $K$.  This supports using a shallow truncation when the
extra unroll depth no longer changes the meta-gradient direction much.
The practical choice $K{=}2$ is therefore made empirically, from the
truncation and return-per-compute studies in the appendix, rather than
from the worst-case constants in the bound.

\begin{corollary}[Truncation depth]
\label{cor:K-threshold}
The truncation contribution in~\eqref{eq:convergence-rate} falls below
any target $\epsilon>0$ once
$K=\Theta(\log(1/\epsilon))$, with constants determined by $B$ and
$\rho$.
\end{corollary}

\subsection{Policy-Utility Residual}
\label{sec:analysis-pur}

The previous results are specific to DRILL-VWI and DRILL-IMG.  To
compare arbitrary world-model objectives, we measure whether the
pixel-MSE gradient is aligned with the direction that improves the
policy after imagined training.  We define the Policy-Utility Residual
as
\begin{equation}
  \PUR(\phi)
  \;=\;
  1 \;-\;
  \frac{
    \bigl\langle
      \nabla_\phi \lossmse(\phi),\;
      -\nabla_\phi J\!\bigl(\pithetap\bigr)
    \bigr\rangle
  }{
    \bigl\lVert \nabla_\phi \lossmse\bigr\rVert
    \,\bigl\lVert \nabla_\phi J\bigr\rVert
  }.
  \label{eq:pur}
\end{equation}
Equivalently, $\PUR=1+\cos\angle(\nabla_\phi\lossmse,\nabla_\phi J)$
whenever both gradients are nonzero.  Thus $\PUR\!\approx\!0$ means
that descending pixel MSE also improves the downstream-return objective,
$\PUR\!\approx\!1$ means the two directions are nearly orthogonal, and
$\PUR\!\approx\!2$ means that descending pixel MSE is anti-aligned with
policy improvement.

We estimate $\nabla_\phi J$ by finite differences in the simulator
rather than by reusing the HVP machinery of DRILL-IMG.  This keeps
\PUR{} independent of any particular training method, so the same
diagnostic applies to WMPO, RLVR-World, DRILL-VWI, DRILL-IMG, and other
world-model objectives.  Section~\ref{sec:exp-pur} tests whether lower
\PUR{} tracks higher downstream success across the methods we evaluate.

%

\section{Experiments}
\label{sec:experiments}

We evaluate DRILL on five robot-manipulation simulators, two VLA
backbones, and five training baselines, while also reporting diagnostic
metrics from policy-aware WM evaluation.  The main text keeps the headline
comparison, compact diagnostic panels, architectural checks, component-level
summaries, and the criterion for choosing DRILL-VWI or DRILL-IMG.  Detailed
ablations, cross-backbone tables, robustness breakdowns, hyperparameter
sweeps, reward-model swaps, and auxiliary diagnostic curves are collected
in Appendices~\ref{app:hparams}--\ref{app:failure-modes}.

\subsection{Setup}
\label{sec:exp-setup}

\para{Benchmarks and baselines.}
We use LIBERO-Long~\citep{libero2023}, CALVIN ABC$\to$D~\citep{calvin2022},
RoboCasa~\citep{robocasa2024}, Meta-World MT50, and SimplerEnv.  These
simulators define both evaluation and the small real-rollout budget used
by DRILL-IMG's outer gradient, $B_{\mathrm{real}}{=}30$ simulator
rollouts per outer step; no real-robot results are included.  Baselines
are Scratch-GRPO, WMPO~\citep{wmpo2025}, VLAW~\citep{vlaw2025},
RLVR-World~\citep{rlvrworld2025}, and a simulator-budget-matched
model-free PPO control, Sim-PPO.  Sim-PPO receives the same simulator
episode budget as DRILL-IMG but no imagined rollouts, isolating whether
the world model amplifies limited simulator interaction.  Policy-Aware
WM~\citep{policyaware2026} is diagnostic rather than a training method,
so we report its ESNR score alongside \PUR{}.

\para{Protocol.}
Headline results use five seeds; ablations use three.  We report Welch's
$t$-test $p$-values and 95\% CIs for head-to-head comparisons.  The
headline run uses about 17k NVIDIA H20-hours, or 24 H20 GPUs for roughly
four weeks.  Appendix~\ref{app:compute} gives the per-phase breakdown.

\subsection{Main Results}
\label{sec:exp-headline}

On $\pi_0$-3B, DRILL-IMG raises mean success rate by 13.1 points over
WMPO and 7.9 points over RLVR-World, averaged across the five benchmarks,
while increasing one-step pixel MSE by 17\% on the same training data.
The gain is largest on long-horizon tasks and smallest on SimplerEnv,
where DRILL-VWI and DRILL-IMG are statistically tied:
$64.0\!\pm\!1.9$ versus $63.1\!\pm\!2.6$.  This exception motivates the
horizon analysis in Section~\ref{sec:exp-v1-v2-gap}.

\begin{table*}[t]
\centering
\scriptsize
\setlength{\tabcolsep}{2.1pt}
\renewcommand{\arraystretch}{0.96}
\resizebox{0.98\textwidth}{!}{%
\begin{tabular}{@{}lcccccccc@{}}
\toprule
\textbf{Method}
 & LIBERO & CALVIN & RoboCasa & MT50 & Simpler
 & MSE $\downarrow$ & ESNR $\uparrow$ & \PUR{} $\downarrow$ \\
\midrule
Scratch-GRPO          & $18.2\!\pm\!2.7$ & $22.5\!\pm\!3.1$ & $\phantom{0}8.1\!\pm\!1.6$ & $21.4\!\pm\!2.4$ & $15.0\!\pm\!2.2$ & N/A          & N/A           & N/A           \\
Sim-PPO               & $28.5\!\pm\!2.1$ & $32.1\!\pm\!2.5$ & $12.4\!\pm\!1.9$ & $31.6\!\pm\!2.0$ & $22.8\!\pm\!1.8$ & N/A          & N/A           & N/A           \\
WMPO                  & $64.2\!\pm\!1.6$ & $71.8\!\pm\!1.3$ & $38.5\!\pm\!2.1$ & $67.4\!\pm\!1.5$ & $55.1\!\pm\!2.0$ & $8.4\!\pm\!0.3$ & $1.00$        & $1.42\!\pm\!0.08$ \\
VLAW                  & $65.1\!\pm\!2.4$ & $74.0\!\pm\!1.8$ & $39.6\!\pm\!2.2$ & $68.2\!\pm\!2.1$ & $57.9\!\pm\!1.7$ & $8.7\!\pm\!0.4$ & $1.09$        & $1.37\!\pm\!0.09$ \\
RLVR-World$^{\dagger}$ & $70.8\!\pm\!1.4$ & $75.3\!\pm\!1.6$ & $42.7\!\pm\!1.8$ & $72.9\!\pm\!1.5$ & $61.6\!\pm\!1.9$ & $\mathbf{7.2\!\pm\!0.3}$ & $1.39$ & $0.97\!\pm\!0.06$ \\
DRILL-VWI$^{\dagger}$ & $76.1\!\pm\!1.2$ & $80.7\!\pm\!1.1$ & $51.8\!\pm\!1.7$ & $75.6\!\pm\!1.4$ & $\mathbf{64.0\!\pm\!1.9}$ & $9.3\!\pm\!0.3$ & $1.68$ & $0.73\!\pm\!0.05$ \\
\textbf{DRILL-IMG}$^{\dagger}$ & $81.3\!\pm\!1.3$ & $83.1\!\pm\!1.4$ & $55.7\!\pm\!2.0$ & $79.4\!\pm\!1.2$ & $63.1\!\pm\!2.6$ & $9.8\!\pm\!0.4$ & $2.09$ & $0.36\!\pm\!0.04$ \\
DRILL-IMG+RLVR-WM$^{\dagger}$ & $\mathbf{83.9\!\pm\!1.1}$ & $\mathbf{85.4\!\pm\!1.2}$ & $\mathbf{59.1\!\pm\!1.6}$ & $\mathbf{81.8\!\pm\!1.3}$ & $67.3\!\pm\!2.1$
                      & $8.6\!\pm\!0.3$ & $\mathbf{2.38}$ & $\mathbf{0.29\!\pm\!0.04}$ \\
\bottomrule
\end{tabular}%
}
\begingroup
\setlength{\abovecaptionskip}{4pt}
\setlength{\belowcaptionskip}{0pt}
\caption{Main results on $\pi_0$-3B.  Values are mean$\pm$std over five
seeds.  Bold marks column best; $^{\dagger}$ denotes $p\!<\!0.05$
versus WMPO.  MSE is in $10^{-3}$ and ESNR is relative to WMPO;
RLVR-WM denotes the RLVR-World WM.}
\label{tab:main-headline}
\vspace{-1.2em}
\endgroup
\end{table*}

The compact diagnostic panels in Figure~\ref{fig:exp-diagnostics} give
three additional views of the same result: DRILL variants break the usual
link between lower pixel MSE and higher success, improve faster on
LIBERO-Long, and reduce \PUR{} during training.  A compute-matched control
in Appendix~\ref{app:main-result-plots} further shows that extra WMPO
updates move LIBERO-Long SR from $64.2\!\pm\!1.6$ to at most
$65.1\!\pm\!1.9$, far below DRILL-IMG's $81.3\!\pm\!1.3$.
Together, these checks rule out a pure compute or pixel-fidelity
explanation: the methods that improve most are those that better align
imagined updates with downstream policy utility.  The appendix reports the
full breakdowns, while the main text keeps only this recurring trend.

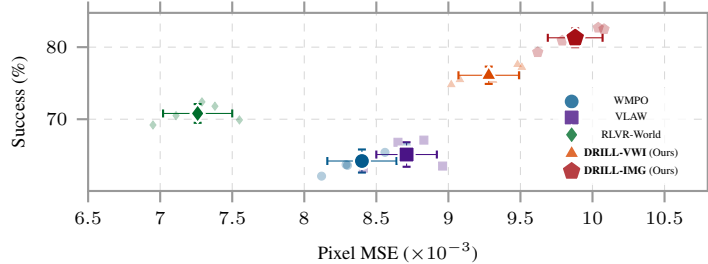
\begin{figure}[t]
\centering
\begin{tikzpicture}
\begin{axis}[
    modern paper,
    width=0.70\linewidth,
    height=1.55in,
    xlabel={Pixel MSE ($\times 10^{-3}$)},
    ylabel={Success (\%)},
    legend pos=south east,
    legend style={nodes={scale=0.58, transform shape}, draw=none, fill=white, fill opacity=0.80},
    ticklabel style={font=\scriptsize},
    label style={font=\scriptsize},
    xmin=6.5, xmax=10.8,
    error bars/error mark options={rotate=90, mark size=1.5pt, line width=0.8pt},
]
\addplot[only marks, mark=*, DeepNavy, opacity=0.3, mark size=1.5pt, forget plot] coordinates {
  (8.12, 62.1) (8.56, 65.4) (8.29, 63.7) (8.71, 66.2) (8.30, 63.6)};
\addplot[only marks, mark=square*, DeepPurple, opacity=0.3, mark size=1.5pt, forget plot] coordinates {
  (8.41, 63.2) (8.83, 67.1) (8.72, 64.9) (8.65, 66.8) (8.96, 63.5)};
\addplot[only marks, mark=diamond*, DeepGreen, opacity=0.3, mark size=1.5pt, forget plot] coordinates {
  (6.95, 69.2) (7.38, 71.8) (7.11, 70.5) (7.29, 72.4) (7.55, 69.9)};
\addplot[only marks, mark=triangle*, BurntOrange, opacity=0.3, mark size=1.5pt, forget plot] coordinates {
  (9.02, 74.8) (9.48, 77.6) (9.31, 75.4) (9.51, 77.2) (9.08, 75.5)};
\addplot[only marks, mark=pentagon*, DeepRed, opacity=0.3, mark size=2pt, forget plot] coordinates {
  (9.62, 79.3) (10.04, 82.7) (9.79, 80.9) (10.08, 82.5) (9.85, 81.1)};

\addplot[only marks, mark=*, DeepNavy, mark options={draw=white, line width=0.8pt, fill=DeepNavy}, mark size=3pt, 
         error bars/.cd, y dir=both, y explicit, x dir=both, x explicit] 
         coordinates { (8.40, 64.2) +- (0.24, 1.6) };
\addlegendentry{WMPO}

\addplot[only marks, mark=square*, DeepPurple, mark options={draw=white, line width=0.8pt, fill=DeepPurple}, mark size=3pt, 
         error bars/.cd, y dir=both, y explicit, x dir=both, x explicit] 
         coordinates { (8.71, 65.1) +- (0.21, 1.7) };
\addlegendentry{VLAW}

\addplot[only marks, mark=diamond*, DeepGreen, mark options={draw=white, line width=0.8pt, fill=DeepGreen}, mark size=3.5pt, 
         error bars/.cd, y dir=both, y explicit, x dir=both, x explicit] 
         coordinates { (7.26, 70.8) +- (0.24, 1.3) };
\addlegendentry{RLVR-World}

\addplot[only marks, mark=triangle*, BurntOrange, mark options={draw=white, line width=0.8pt, fill=BurntOrange}, mark size=3.5pt, 
         error bars/.cd, y dir=both, y explicit, x dir=both, x explicit] 
         coordinates { (9.28, 76.1) +- (0.21, 1.2) };
\addlegendentry{\textbf{DRILL-VWI} (Ours)}

\addplot[only marks, mark=pentagon*, DeepRed, mark options={draw=white, line width=0.8pt, fill=DeepRed}, mark size=4pt, 
         error bars/.cd, y dir=both, y explicit, x dir=both, x explicit] 
         coordinates { (9.88, 81.3) +- (0.19, 1.3) };
\addlegendentry{\textbf{DRILL-IMG} (Ours)}

\end{axis}
\end{tikzpicture}
\vspace{-1.2em}
\caption{MSE--success decoupling on LIBERO-Long ($\pi_0$-3B). Faint markers are seeds; bold markers are mean$\pm$std.}
\vspace{-1.2em}
\label{fig:mse-vs-sr}
\end{figure}

\begin{figure*}[t]
\centering
\begin{minipage}[t]{0.32\textwidth}
\vspace{0pt}
\centering
\textbf{(a) Learning curves}\par\vspace{1pt}
\begin{tikzpicture}
\begin{axis}[
    modern paper,
    width=\linewidth,
    height=2.0in,
    xlabel={Outer-loop Step},
    ylabel={Policy Success Rate (\%)},
    xmode=log,
    grid=major,
    legend pos=south east,
    legend style={nodes={scale=0.70, transform shape}},
    log basis x=10,
]
\addplot[thick, color=DeepNavy, mark=*, mark options={fill=white, solid}, mark size=1.2pt] coordinates {
  (10, 20) (30, 34) (100, 45) (300, 55) (1000, 60) (3000, 62.7) (10000, 64.2)};
\addlegendentry{WMPO}

\addplot[thick, color=DeepGreen, mark=diamond*, mark options={fill=white, solid}, mark size=1.4pt] coordinates {
  (10, 22) (30, 34) (100, 52) (300, 61) (1000, 66.5) (3000, 69.4) (10000, 70.8)};
\addlegendentry{RLVR}

\addplot[thick, color=BurntOrange, mark=triangle*, mark options={fill=white, solid}, mark size=1.4pt] coordinates {
  (10, 20) (30, 40) (100, 57) (300, 69) (1000, 74.3) (3000, 75.8) (10000, 76.1)};
\addlegendentry{\textbf{DRILL-VWI}}

\addplot[thick, color=DeepRed, mark=star, mark size=1.6pt] coordinates {
  (10, 22) (30, 42) (100, 61) (300, 73) (1000, 78.6) (3000, 80.5) (10000, 81.3)};
\addlegendentry{\textbf{DRILL-IMG}}

\addplot[thick, color=DeepRed!50!black, mark=pentagon*, mark options={fill=white, solid}, mark size=1.4pt, dashed] coordinates {
  (10, 26) (30, 46) (100, 66) (300, 77) (1000, 81.4) (3000, 83.2) (10000, 83.9)};
\addlegendentry{\textbf{+ RLVR}}
\end{axis}
\end{tikzpicture}
\end{minipage}\hfill
\begin{minipage}[t]{0.32\textwidth}
\vspace{0pt}
\centering
\textbf{(b) Horizon effect}\par\vspace{1pt}
\begin{tikzpicture}
\begin{axis}[
    modern paper,
    width=\linewidth,
    height=2.0in,
    xlabel={Rollout Horizon $H$},
    ylabel={Policy Success Rate (\%)},
    grid=major,
    legend pos=south west,
    legend style={nodes={scale=0.70, transform shape}},
    xtick={20, 50, 100, 200},
    ymin=40, ymax=90,
]

\addplot[thick, color=DeepNavy, mark=*, mark options={fill=white, solid}, mark size=1.2pt] coordinates {
  (20, 67.4) (50, 64.2) (100, 58.0) (200, 49.8)};
\addlegendentry{WMPO}

\addplot[thick, color=BurntOrange, mark=triangle*, mark options={fill=white, solid}, mark size=1.4pt] coordinates {
  (20, 72.6) (50, 76.1) (100, 75.9) (200, 70.5)};
\addlegendentry{\textbf{DRILL-VWI}}

\addplot[thick, color=DeepRed, mark=star, mark size=1.6pt] coordinates {
  (20, 74.4) (50, 81.3) (100, 83.4) (200, 78.6)};
\addlegendentry{\textbf{DRILL-IMG}}

\addplot[thick, dashed, color=black, sharp plot] coordinates {
  (20, 73.1) (50, 76.4) (100, 79.8) (200, 83.7)};
\addlegendentry{VWI + fitted locality-residual trend}

\end{axis}
\end{tikzpicture}
\end{minipage}\hfill
\begin{minipage}[t]{0.32\textwidth}
\vspace{0pt}
\centering
\textbf{(c) \PUR{} dynamics}\par\vspace{1pt}
\begin{tikzpicture}
\begin{axis}[
    modern paper,
    width=\linewidth,
    height=2.0in,
    xlabel={Outer-loop Step},
    ylabel={\PUR{}},
    xmode=log,
    legend pos=south west, 
    legend style={nodes={scale=0.70, transform shape}, fill=white, fill opacity=0.85, draw=none}
]
\addplot[thick, color=DeepNavy, mark=*, mark options={fill=white, solid}, mark size=1.2pt] coordinates {
  (10, 0.98) (30, 1.02) (100, 1.21) (300, 1.27) (1000, 1.39) (3000, 1.38) (10000, 1.44)};
\addlegendentry{WMPO}

\addplot[thick, color=DeepGreen, mark=diamond*, mark options={fill=white, solid}, mark size=1.4pt] coordinates {
  (10, 0.89) (30, 0.97) (100, 0.92) (300, 0.99) (1000, 0.94) (3000, 0.98) (10000, 0.97)};
\addlegendentry{RLVR-World}

\addplot[thick, color=BurntOrange, mark=triangle*, mark options={fill=white, solid}, mark size=1.4pt] coordinates {
  (10, 1.12) (30, 1.03) (100, 0.88) (300, 0.82) (1000, 0.76) (3000, 0.74) (10000, 0.73)};
\addlegendentry{\textbf{DRILL-VWI}}

\addplot[thick, color=DeepRed, mark=star, mark size=1.6pt] coordinates {
  (10, 1.08) (30, 0.98) (100, 0.81) (300, 0.58) (1000, 0.48) (3000, 0.39) (10000, 0.36)};
\addlegendentry{\textbf{DRILL-IMG}}
\end{axis}
\end{tikzpicture}
\end{minipage}
\vspace{-2pt}
\caption{Three compact diagnostics on LIBERO-Long. (a) DRILL variants pass baselines early and continue improving. (b) Larger rollout horizon widens the DRILL-IMG--DRILL-VWI gap until imagined trajectories become noisy. (c) Lower \PUR{} tracks higher policy utility during training.}
\vspace{-1.2em}
\label{fig:exp-diagnostics}
\label{fig:training-curves}
\label{fig:img-vwi-gap}
\label{fig:pur-curves}
\end{figure*}
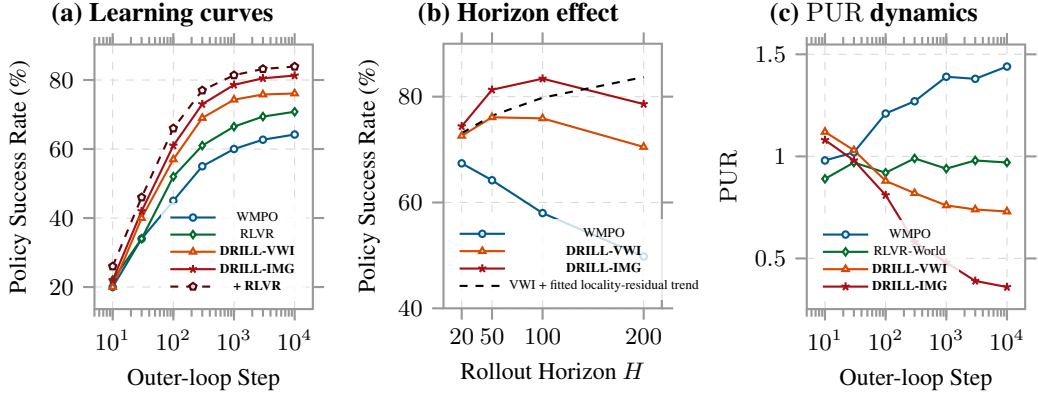

\para{Backbone and pretraining checks.}
\label{sec:exp-openvla}
We repeat the headline comparison on OpenVLA-7B using LIBERO-Long,
CALVIN, and RoboCasa.  LIBERO-Long and CALVIN preserve the ranking
DRILL-IMG $>$ DRILL-VWI $>$ RLVR-World $>$ WMPO.  On RoboCasa,
DRILL-IMG and DRILL-VWI are statistically tied,
$50.8\!\pm\!2.3$ versus $48.7\!\pm\!2.1$ with $p\!=\!0.12$, matching
that benchmark's shorter, object-centric episodes.  DRILL-IMG reaches
68.9 mean SR on OpenVLA, compared with 53.8 for WMPO, a 15.1-point gain;
the full per-benchmark table is in Appendix~\ref{app:openvla-replication}.
This replication changes both the policy backbone and the pretraining
mixture, yet preserves the same method ordering on the two longer-horizon
benchmarks.  We therefore use the appendix tables mainly as robustness
evidence for the objective, rather than as a separate architecture claim.

\para{Orthogonality with RLVR-World.}
\label{sec:exp-orthogonality}
Initializing the WM with RLVR-World pretraining and then applying
DRILL-IMG improves mean SR by another roughly 3 points over DRILL-IMG alone.
Prediction-correctness pretraining and downstream-return imagination are
therefore complementary rather than interchangeable.  The stacked variant
helps most on SimplerEnv, where the shorter horizon weakens the
DRILL-IMG HVP signal and a better pretrained WM matters more.

A plausible mechanism is that RLVR-World pretraining lowers the variance of
the supervised prediction gradient on task-relevant transitions, so that the
downstream-return outer gradient is less likely to be drowned out early in
training. We treat this as an interpretation of the empirical trend rather
than a formal claim, and leave a direct curvature-level test
(cf.\ Appendix~\ref{app:exp-hvp-numerics}) to future work.

\subsection{Ablations and Scaling}
\label{sec:exp-ablations}

\para{Objective ablations.}
Across LIBERO, CALVIN, and RoboCasa, each factor in the signed DRILL-VWI
coaching surrogate helps: visitation-only reaches 64.2 mean SR,
signed-advantage-only 66.8, and confidence-only 60.1, compared with 58.2
for WMPO.  Combining all three gives 69.2.  The exponent sweep peaks at
\(\alpha{=}1\), consistent with Proposition~\ref{prop:v1-is-v2-first-order}:
the score-compatible local surrogate is linear in the signed advantage,
whereas \(\alpha\ne1\) deliberately distorts that direction. Increasing
to \(\alpha{=}2\) lowers SR to 67.9, consistent with higher variance from
overemphasizing high-magnitude signed-advantage samples. A minimal
value-aware baseline, VA-WM, reaches 62.4 mean SR, closing part of the gap
to WMPO but remaining 6.8 points below DRILL-VWI and 11.0 below
DRILL-IMG; the complete ablation table is in
Appendix~\ref{app:objective-ablations}.

\para{Scaling and coupling.}
On LIBERO-Long, the DRILL-IMG margin over WMPO grows with WM size: 7.9
points at 300M parameters, 13.1 at 1.2B, and 16.4 at 7B.  DRILL-IMG also
transfers across WM backbones: Ctrl-World obtains $81.3\!\pm\!1.3$ SR and
a matched Cosmos-Predict2-style model obtains $79.6\!\pm\!1.7$, a gap
within two standard deviations.  Finally, freezing the VWI signed coaching multipliers from a
GRPO-warmstart policy reaches only $64.9\!\pm\!1.5$ mean SR, 4.3 points
below DRILL-VWI.  The gain therefore comes from tracking the current
policy's signed advantage signal, not from a fixed supervised MSE
reweighting.

\subsection{When DRILL-IMG Pays Off over DRILL-VWI}
\label{sec:exp-v1-v2-gap}

At the default horizon \(H{=}50\), DRILL-VWI captures most of the gain on
LIBERO-Long: it recovers 11.9 of the 17.1 points between WMPO and
DRILL-IMG.  DRILL-IMG costs about \(1.4\times\) more per outer step and
uses \(B_{\mathrm{real}}{=}30\) simulator rollouts that DRILL-VWI does
not require, so the choice should depend on horizon.  The residual
decomposition suggests this
dependence: the one-step DRILL-VWI surrogate omits multi-step
trajectory-score terms collected in the trajectory-locality residual
\(R_H\), while DRILL-IMG partially recovers these terms through
truncated BPTT.  If one additionally assumes
\(R_H=\mathcal{O}(\eta_\theta H^2)\), this residual induces an
\(\mathcal{O}(\eta_\theta^2H^2)\) horizon-dependent component, but the
experiment should be read as testing whether the IMG--VWI gap grows
with \(H\), not as isolating a universal \(H^2\) law.

The horizon sweep (Figure~\ref{fig:img-vwi-gap}) is consistent with this residual-based criterion.  The DRILL-IMG margin over
DRILL-VWI grows from 1.8 points at $H{=}20$ to 5.2 at $H{=}50$, 7.5 at
$H{=}100$, and 8.1 at $H{=}200$.  WMPO falls from 67.4 to 49.8 SR as
rollout error compounds; DRILL-VWI is more stable but saturates around
the mid-70s.  DRILL-IMG peaks at $H{=}100$ with 83.4 SR and drops to 78.6
at $H{=}200$, where imagined trajectories become too noisy.  The
short-horizon checks point the other way: on LIBERO-Long the IMG-over-VWI
gap shrinks to 0.6 at $H{=}10$ and becomes a non-significant
$-0.2\!\pm\!1.1$ at $H{=}5$; on SimplerEnv, whose effective horizon is
below $H{=}20$, DRILL-VWI slightly exceeds DRILL-IMG, 64.0 versus 63.1 SR.
We recommend DRILL-VWI as the budget-conscious choice for very short
horizons or when simulator rollouts are unavailable, and DRILL-IMG when
the task horizon is moderate to long or when the extra simulator budget is
acceptable. In our default $H{=}50$ setting, DRILL-IMG already gives a
5.2-point gain over DRILL-VWI on LIBERO-Long, while SimplerEnv illustrates
the short-horizon regime where VWI can be preferable.  Full horizon tables and
reward-model swaps are in Appendices~\ref{app:horizon-sweep} and
\ref{app:exp-reward-swap}.

\para{HVP validation.}
\label{sec:exp-hvp}
On a 2D gridworld and Meta-World MT10, where the exact meta-gradient is
tractable, the \(K{=}2\) truncated HVP has cosine similarity
\(0.74\!\pm\!0.04\) and \(0.71\!\pm\!0.05\) to the untruncated target
meta-gradient, respectively. The full $K$-sweep and estimated contraction rate $\rho$
are in Appendix~\ref{app:exp-hvp-numerics}.

\subsection{Diagnostics and Robustness}
\label{sec:exp-diagnostics}

\para{Robustness and diagnostics.}
\label{sec:exp-ood}
\label{sec:exp-pur}
\label{sec:exp-rollout-quality}
\label{sec:exp-human-eval}
Under LIBERO zero-shot perturbations, DRILL-IMG has an 18.6\% average
relative SR drop, compared with 30.2\% for WMPO and 25.6\% for
RLVR-World.  The only exception is pose perturbation, where DRILL-VWI
slightly wins, again matching a shorter effective horizon; the
per-condition breakdown is in Appendix~\ref{app:exp-ood}.  \PUR{} also
tracks policy utility (Figure~\ref{fig:pur-curves}): WMPO drifts upward,
RLVR-World plateaus, DRILL-VWI stays bounded, and DRILL-IMG decreases;
across the \PUR{}-reporting headline rows, the Spearman correlation with
downstream SR is $\rho_S=-0.89$ with 95\% bootstrap CI $[-0.97,-0.71]$.
Two held-out variants give 1.5 points mean absolute error when final SR is
predicted from checkpoint-level \PUR{} at step 5k.

DRILL also improves task-relevant imagination rather than visual fidelity
alone.  On LIBERO-Long, WMPO reaches lower pixel MSE but plateaus at
$0.61\!\pm\!0.02$ TRPS, while DRILL-IMG reaches
$0.81\!\pm\!0.03$ TRPS despite higher MSE.  In a forced-choice human
study with 120 trials and six blinded raters, visual-quality preferences
for WMPO or RLVR-World over DRILL-IMG are not significant, but when task
completion is shown, 62\% prefer DRILL-IMG over WMPO
($p\!<\!0.001$).  TRPS curves and human-evaluation details are in
Appendices~\ref{app:pur-diagnostics} and~\ref{app:exp-human-eval}.


\section{Conclusion}
\label{sec:conclusion}

We introduced DRILL, a bilevel framework for training video world models
as policy-improvement environments rather than as standalone next-frame
predictors.  Its two instantiations expose a practical trade-off:
DRILL-VWI provides a lightweight visitation- and advantage-weighted
surrogate that can be used in a WMPO-style pipeline, while DRILL-IMG more
directly optimizes downstream return by differentiating through imagined
policy updates.  The theoretical connection between these variants
explains when the cheaper local surrogate should be sufficient and when
longer-horizon trajectory effects make the meta-gradient worth its extra
cost.
Across five manipulation simulators and two VLA backbones, DRILL improves
policy success even when one-step pixel MSE becomes worse, reinforcing
the central claim that visual fidelity is not the right standalone target
for world-model post-training.  The accompanying \PUR{} diagnostic makes
this mismatch measurable, and the gains from combining DRILL with
RLVR-World suggest that prediction-correctness pretraining and
policy-aware downstream optimization are complementary.  More broadly,
DRILL shifts the optimization target from making imagined videos look
accurate to making learning inside those videos useful.

\paragraph{Limitations.}
Our evaluation remains simulation-only and uses a small simulator-rollout
budget for DRILL-IMG rather than physical robot interaction.  This leaves
open how sensitive the method is to real-world perception noise,
actuation delay, safety constraints, and sim-to-real distribution shift.
DRILL-IMG also introduces additional compute and memory overhead through
truncated backpropagation and Hessian-vector products, so DRILL-VWI may
be the more appropriate choice when interaction, memory, or wall-clock
budget is the binding constraint.  Extending DRILL to real-robot
post-training and reducing the cost of the meta-gradient estimator are
therefore important next steps.

\bibliographystyle{unsrt}
\bibliography{refs}

\appendix
\renewcommand{\thesection}{\Alph{section}}
\setcounter{section}{0}

%
%
%

\section{Notation and Preliminaries}
\label{app:notation}

This appendix collects the symbols, assumptions, and intermediate
definitions used throughout the paper, so that
Appendices~\ref{app:prop-v1-v2}--\ref{app:thm-convergence} can be read
without re-tracing the main text.

\subsection{Symbol glossary}
\label{app:notation-symbols}

\begingroup
\small
\renewcommand{\arraystretch}{1.04}
\begin{longtable}{@{}>{\raggedright\arraybackslash}p{0.39\textwidth}@{\hspace{0.02\textwidth}}>{\raggedright\arraybackslash}p{0.54\textwidth}@{}}
\caption{Notation used throughout the paper.}
\label{tab:notation}\\
\toprule
\textbf{Symbol} & \textbf{Meaning} \\
\midrule
\endfirsthead
\caption[]{Notation used throughout the paper (continued).}\\
\toprule
\textbf{Symbol} & \textbf{Meaning} \\
\midrule
\endhead
\midrule
\multicolumn{2}{r@{}}{\footnotesize Continued on next page.}\\
\endfoot
\bottomrule
\endlastfoot
\multicolumn{2}{@{}l}{\textbf{Observations, actions, language}} \\
$o_t \in \mathcal{O}$           & Observation (image) at time $t$. \\
$o_{\le t}$                     & Observation history $(o_0, \dots, o_t)$. \\
$a_t \in \mathcal{A}$           & Action at time $t$. \\
$a_{t:t+H_a}$                   & $H_a$-step action chunk produced by $\pitheta$. \\
$\ell$                          & Natural-language instruction. \\
$r_t \in [0, 1]$                & Reward at time $t$. \\
\midrule
\multicolumn{2}{@{}l}{\textbf{Models}} \\
$\pitheta(a_{t:t+H_a} \mid o_{\le t}, \ell)$
                                & VLA policy with parameters $\theta$. \\
$\Wphi(o_{t+1} \mid o_{\le t}, a_t)$
                                & Video world model with parameters $\phi$. \\
$\hat o_{t+1}(\phi)$            & Predicted next-frame mean of $\Wphi$. \\
$r_\psi(o_t, \ell)\in[0,1]$     & Frozen vision-language reward model. \\
\midrule
\multicolumn{2}{@{}l}{\textbf{Distributions and data}} \\
$\mathcal{D}_{\mathrm{demo}}$   & Offline demonstration dataset. \\
$P_{\mathrm{env}}(\pi)$         & Target-environment trajectory distribution under $\pi$
                                  (in our experiments, the high-fidelity simulator). \\
$d_{\rm env}^{\pi_\theta}(o_t)$ & On-policy state-visitation marginal of $\pitheta$ in $P_{\mathrm{env}}$. \\
$\tilde d_{\phi}^{\pi_\theta}(o_t)$ & Time-averaged visitation marginal induced by rolling out $\pitheta$ inside $\Wphi$. \\
$\widehat{\tilde d}_{\phi}^{\pi_\theta}(o_t)$ & Minibatch-normalized estimator of $\tilde d_{\phi}^{\pi_\theta}(o_t)$ used in DRILL-VWI. \\
$\mu_{\mathrm{roll}}$            & Local one-step imagined-rollout transition distribution induced by
                                  $o_t\sim \tilde d_{\phi}^{\pi_\theta}$, $a_t\sim\pi_\theta(\cdot\mid o_{\le t})$,
                                  and $o_{t+1}\sim\Wphi(\cdot\mid o_{\le t},a_t)$. \\
$\mu_{\mathrm{demo}}$            & Raw demonstration transition distribution before applying
                                  VWI weights. \\
$\hA(o_t, a_t)$                 & Group-relative GRPO advantage, treated as a stop-gradient. \\
$\tilde A_\alpha(o_t,a_t)$      & Signed advantage power
                                  $\operatorname{sgn}(\hA(o_t,a_t))|\hA(o_t,a_t)|^\alpha$. \\
$s(o_t)>0$                      & Positive WM-ensemble confidence / stability factor. \\
$w(o_t,a_t)$                    & Unclipped signed DRILL-VWI coefficient using
                                  $\widehat{\tilde d}_{\phi}^{\pi_\theta}(o_t)$, Eq.~\eqref{eq:vwi-weight}. \\
$\bar w(o_t,a_t)$ & Clipped, normalized, stop-gradient DRILL-VWI coefficient
                  used by the implemented DRILL-VWI objective. \\
$\mathcal B$                    & Minibatch used to normalize $\bar w$. \\
\midrule
\multicolumn{2}{@{}l}{\textbf{Objectives and updates}} \\
$J(\pi)$                        & True expected return
                                  $\E_{\tau\sim P_{\mathrm{env}}(\pi)}[\sum_t r_t]$. \\
$\hJ_{\mathrm{GRPO}}(\pi; \mathcal{W})$
                                & GRPO surrogate when $\pi$ is rolled out in model $\mathcal{W}$. \\
$\theta^{+}(\phi)$              & Inner-update policy
                                  $\theta + \eta_\theta\nabla_\theta
                                  \hJ_{\mathrm{GRPO}}(\pitheta; \Wphi)$. \\
$\lossmse(\phi)$                & Pixel mean-squared error,
                                  Eq.~\eqref{eq:wmpo-mse}. \\
$\lossvwi(\phi;\bar w)$         & DRILL-VWI signed local coaching surrogate,
                                  Eq.~\eqref{eq:vwi-loss}. \\
$\lossvwi^{\mathrm{core}}(\phi)$
                                & Idealized DRILL-VWI surrogate with exact imagined-rollout visitation,
                                  $\alpha=1$, $s(o_t)=1$, and unclipped weights. \\
$\nabla_\phi J(\pithetap)$      & Target meta-gradient underlying DRILL-IMG,
                                  Eq.~\eqref{eq:meta-gradient}. \\
$q_\theta(o_t,a_t)$             & Local policy-score contraction,
                                  Eq.~\eqref{eq:q-theta-app}. \\
$R_H$                           & Trajectory-locality residual from omitted multi-step
                                  trajectory-score terms. \\
$R_{\mathrm{demo}}$             & Demo-anchor mismatch residual,
                                  Eq.~\eqref{eq:demo-anchor-residual-app}. \\
$\epsilon_{\mathrm{comp}}$      & Score-compatibility residual,
                                  Eq.~\eqref{eq:compat-residual-app}. \\
$\PUR(\phi)$                    & Policy-Utility Residual,
                                  Eq.~\eqref{eq:pur}. \\
\midrule
\multicolumn{2}{@{}l}{\textbf{Hyperparameters}} \\
$H$                             & Rollout horizon (imagined trajectory length). \\
$H_a$                           & Action-chunk length output by $\pitheta$. \\
$K$                             & Truncated-BPTT depth in DRILL-IMG. \\
$\eta_\theta > 0$               & Inner-loop policy step size. \\
$\eta_\phi   > 0$               & Outer-loop world-model step size. \\
$\alpha \ge 0$                  & DRILL-VWI advantage exponent (default $\alpha=1$). \\
$w_{\max}$                      & Signed-weight clipping threshold for DRILL-VWI. \\
$B_{\mathrm{real}}$             & Per-step simulator-rollout batch in DRILL-IMG
                                  (default 30). \\
\midrule
\multicolumn{2}{@{}l}{\textbf{Theory constants}} \\
$\sigma_W$                      & Standard deviation of the Gaussian WM head. \\
$M_W$                           & Uniform bound on the per-sample WM score
                                  $\|\nabla_\phi\log\Wphi(o_{t+1}\mid o_{\le t},a_t)\|$. \\
$\kappa>0$                      & Directionally compatible scalar in the approximation
                                  $q_\theta(o_t,a_t)\approx\kappa$ on the
                                  advantage-weighted support. \\
$\sigma_g^2$                    & Conditional variance bound for the noise in the
                                  $K$-truncated meta-gradient estimator, after
                                  subtracting its conditional mean (A4). \\
$L$                             & Lipschitz-smoothness constant of the fixed conditional
                                  objective $J_{\bar\theta}$ in $\phi$ (A5). \\
$B>0,\rho\in(0,1)$              & Truncation-bias decay constants
                                  $\lVert\mathrm{bias}_K\rVert\le B\rho^K$ (A6). \\
$T$                             & Total outer iterations. \\
$J^{*}, J_0$                    & Optimal and initial values of the fixed conditional
                                  objective $J_{\bar\theta}$. \\
$\Delta$                        & Initial optimality gap $J^*-J_0$. \\
\end{longtable}
\endgroup

\FloatBarrier

\subsection{Assumptions}
\label{app:notation-assumptions}

Six assumptions are used across the proofs in
Appendices~\ref{app:prop-v1-v2}--\ref{app:thm-convergence}.  We
restate them here in compact form; the long-form versions, with
remarks on tightness and on the WM-architecture cases they cover,
appear inline at first use in each proof.

\begin{itemize}[leftmargin=*, itemsep=2pt, topsep=2pt]
  \item \textbf{(A1) Outer smoothness in $\theta$.}
        $J:\Theta\to\mathbb{R}$ is $C^2$ in a neighborhood of
        $\theta$ with bounded second derivatives, so that
        \[
          J(\theta+\eta_\theta g)
          =
          J(\theta)
          +
          \eta_\theta\nabla_\theta J(\theta)^\top g
          +
          \mathcal{O}(\eta_\theta^2).
        \]

  \item \textbf{(A2) Gaussian WM head.}
        For each $(o_{\le t},a_t)$, the WM output distribution is
        Gaussian with mean $\hat o_{t+1}(\phi)$ and isotropic
        covariance $\sigma_W^2 I$, and $\phi\mapsto\hat o_{t+1}(\phi)$
        is $C^1$.  Equivalently,
        \[
          \nabla_\phi\log\Wphi(o_{t+1}\mid o_{\le t},a_t)
          =
          -\frac{1}{2\sigma_W^2}
          \nabla_\phi
          \|\hat o_{t+1}(\phi)-o_{t+1}\|_2^2 .
        \]
        For diffusion or flow WMs, an analogous score-objective version
        can be derived, but the simple pixel-MSE identity above is
        specific to the Gaussian-head presentation; see the remark in
        Appendix~\ref{app:prop-setup}.

  \item \textbf{(A3) One-step on-policy stationarity and stop-gradient
        advantage.}
        Over one inner update, the imagined-rollout visitation marginal
        $\tilde d_{\phi}^{\pi_\theta}$ is treated as locally fixed.  The induced
        occupancy drift generally carries horizon- and mixing-dependent constants,
        and is absorbed into the residual terms of
        Proposition~\ref{prop:v1-is-v2-first-order}.  The
        group-relative GRPO advantage $\hA(o_t,a_t)$ is treated as a
        stop-gradient, and the expected inner update is analyzed in the
        vanilla GRPO regime where the importance-ratio clip is inactive:
        \[
          \nabla_\theta\hJ_{\mathrm{GRPO}}
          =
          \E_{o_t\sim \tilde d_{\phi}^{\pi_\theta},\,a_t\sim\pi_\theta(\cdot\mid o_{\le t})}
          \left[
            \hA(o_t,a_t)
            \nabla_\theta\log\pitheta(a_t\mid o_{\le t})
          \right].
        \]
        In practice, clip activation and advantage-estimation drift are
        absorbed into the local residuals in
        Proposition~\ref{prop:v1-is-v2-first-order}; imagined-visitation drift
        contributes to the trajectory-locality residual $R_H$, and
        demo-anchored transfer effects contribute to
        $R_{\mathrm{demo}}$.

  \item \textbf{(A4) Bounded truncated meta-gradient variance.}
        For the fixed conditional objective $J_{\bar\theta}$, the
        $K$-truncated meta-gradient estimator
        $\hat g_t := \widehat{\nabla}_\phi J_{\bar\theta}(\phi_t)$ from
        $B_{\mathrm{real}}$ simulator rollouts and one truncated HVP
        satisfies
        \[
          \E\!\left[
            \|\hat g_t-\E[\hat g_t\mid \phi_t]\|^2
            \,\middle|\, \phi_t
          \right]
          \le
          \sigma_g^2,
        \]
        uniformly in $t$.

  \item \textbf{(A5) $L$-smoothness in $\phi$.}
        For fixed current policy parameters $\bar\theta$, the
        conditional objective $J_{\bar\theta}$ is $L$-smooth along the
        outer optimization trajectory:
        \[
          \|\nabla_\phi J_{\bar\theta}(\phi')-
            \nabla_\phi J_{\bar\theta}(\phi)\|
          \le
          L\|\phi'-\phi\|
        \]
        for all $\phi,\phi'$ on that trajectory.

  \item \textbf{(A6) Exponential truncation-bias decay.}
        The bias of the $K$-truncated meta-gradient estimator satisfies
        \[
          \bigl\|
            \E[\hat g_t\mid\phi_t]-\nabla_\phi J_{\bar\theta}(\phi_t)
          \bigr\|
          \le
          B\rho^K
        \]
        for some $B>0$ and $\rho\in(0,1)$ determined by the spectral
        radius of the inner-loop Jacobian.  This is the standard
        contraction assumption used for truncated BPTT through unrolled
        iterative solvers, and parallels the spectral-radius bounds
        that drive the implicit-MAML rate of
        Rajeswaran et al.~\citep{imaml2019}; we estimate $\rho$ directly
        in Section~\ref{sec:exp-hvp}.

\end{itemize}

(A1)--(A3) underlie Proposition~\ref{prop:v1-is-v2-first-order}
(Appendix~\ref{app:prop-v1-v2}); (A4)--(A6) underlie
Theorem~\ref{thm:v2-convergence}
(Appendix~\ref{app:thm-convergence}).  Appendix~\ref{app:assumption-constants}
reports how the constants $L$, $\sigma_g$, $B$, and $\rho$ are
estimated in the experiments and gives the numerical ranges used
to instantiate the theorem.

\paragraph{Remark on the score-compatibility residual.}
Proposition~\ref{prop:v1-is-v2-first-order} does not assume that
DRILL-VWI is exactly equal to the target first-order direction
underlying DRILL-IMG.  Instead, it keeps the local policy-score
contraction
\[
  q_\theta(o_t,a_t)
  =
  \left\langle
    \nabla_\theta\log\pitheta(a_t\mid o_{\le t}),
    \nabla_\theta J(\pitheta)
  \right\rangle
\]
explicit, and measures the quality of the DRILL-VWI substitution
\(q_\theta(o_t,a_t)\approx\kappa\) through
\(\epsilon_{\mathrm{comp}}\).  The DRILL-VWI--DRILL-IMG gap therefore
has a compatibility component, a trajectory-locality component
\(R_H\), a demo-anchor component \(R_{\mathrm{demo}}\), and a
higher-order Taylor component.

\paragraph{Remark on (A6) and its failure mode.}
The contraction parameter $\rho$ is the spectral radius of the
inner-loop Jacobian
$\partial \theta^{+}(\phi) / \partial \theta$ evaluated along
the unrolled imagined rollout.  For a stable inner-loop optimizer
on a well-conditioned WM, $\rho<1$ holds locally; we measure
$\rho$ directly on gridworld and Meta-World MT10 in
Section~\ref{sec:exp-hvp}.  The assumption breaks when the inner loop
sits near a fixed point of the WM dynamics, so that successive
imagined transitions become near-identity in $\theta$: the
$K$-truncated meta-gradient estimator then has $\rho\to1$ and the
$\rho^{2K}$ term in Theorem~\ref{thm:v2-convergence}
no longer decays meaningfully in $K$.  Engineering remedies are spectral
normalization on the WM head, weight decay on $\phi$, and diagonal
damping of the inner Jacobian; we use weight decay throughout
(Appendix~\ref{app:hparams}) and observe no $\rho\to1$ regime in
our experiments, but flag the regime explicitly because the
convergence guarantee is no longer informative in it.

\subsection{The GRPO surrogate, restated}
\label{app:notation-grpo}

For self-containment, we restate the GRPO
surrogate~\citep{wmpo2025} used inside the inner update.  Given
an imagined trajectory
$\tau=(o_0,a_0,r_0,o_1,\dots,o_H)$ sampled by $\pitheta$ under
$\Wphi$,
\begin{equation}
  \hJ_{\mathrm{GRPO}}(\pitheta; \Wphi)
  =
  \E_{\tau=(o_0,a_0,\ldots,o_H)\sim P_{\Wphi}(\pitheta)}
  \left[
    \frac{1}{H}\sum_{t=0}^{H-1}
    \hA(o_t,a_t)\log \pitheta(a_t\mid o_{\le t})
  \right],
  \label{eq:grpo-surrogate-appx}
\end{equation}
where $P_{\Wphi}(\pitheta)$ denotes the imagined trajectory
law induced by rolling out $\pitheta$ in $\Wphi$, and
$\hA(o_t,a_t)$ is the group-relative advantage formed by sampling a
group of $G$ rollouts from $\pitheta$ and standardizing their returns
within the group.  The advantage is treated as a constant in $\theta$
and $\phi$ for the purposes of differentiation, in line with the
stop-gradient GRPO implementation.  Residuals from advantage
recomputation and policy-ratio clipping are included in the local
approximation residuals of Proposition~\ref{prop:v1-is-v2-first-order};
one-step imagined-visitation drift contributes to $R_H$, and transferring
rollout-derived advantages to raw demonstration transitions contributes
to $R_{\mathrm{demo}}$.  All other appearances of
$\hJ_{\mathrm{GRPO}}$ in this paper refer to
Eq.~\eqref{eq:grpo-surrogate-appx}.

%
%

\section{Proof of Proposition~\ref{prop:v1-is-v2-first-order}:
First-order meta-gradient and the DRILL-VWI surrogate}
\label{app:prop-v1-v2}

This appendix proves Proposition~\ref{prop:v1-is-v2-first-order}
from Section~\ref{sec:analysis-prop}.  The proof derives the target
first-order direction underlying DRILL-IMG, keeps the local
policy-score contraction explicit, and shows that DRILL-VWI is
recovered as a score-compatible closed-form surrogate with explicit
residual terms.
Thus DRILL-VWI is a principled local gradient surrogate rather than
an exact first-order identity in general.

\subsection{Setup and notation}
\label{app:prop-setup}

We recall the bi-level formulation from Section~\ref{sec:framework-setup}.
The VLA policy $\pitheta(a_{t:t+H_a} \mid o_{\le t}, \ell)$ is
parameterized by $\theta$, and the video world model (WM)
$\Wphi(o_{t+1} \mid o_{\le t}, a_t)$ is parameterized by $\phi$.
The reward model $r_\psi$ is frozen throughout.  Given the current
$(\theta,\phi)$, the inner one-step GRPO update on imagined
trajectories generated by $\Wphi$ is
\begin{equation}
    \theta^+(\phi)
    =
    \theta + \eta_\theta g(\phi,\theta),
    \qquad
    g(\phi,\theta)
    :=
    \nabla_\theta \hat J_{\mathrm{GRPO}}\bigl(\pitheta;\Wphi\bigr),
    \label{eq:inner-update-appx}
\end{equation}
where $\eta_\theta>0$ is the inner learning rate and
$\hat J_{\mathrm{GRPO}}$ is the GRPO surrogate evaluated on rollouts
from $\Wphi$.  The outer objective is the expected return of the
updated policy in the target environment,
\begin{equation}
    J\bigl(\pi_{\theta^+(\phi)}\bigr)
    =
    \E_{\tau\sim P_{\mathrm{env}}(\pi_{\theta^+(\phi)})}
    \left[
      \textstyle\sum_t r_t
    \right].
    \label{eq:outer-objective-appx}
\end{equation}
DRILL-IMG optimizes $\phi$ by differentiating
\eqref{eq:outer-objective-appx} through
\eqref{eq:inner-update-appx}.  DRILL-VWI instead uses a signed local
coaching surrogate
\begin{equation}
    \lossvwi(\phi;\bar w)
    =
    \E_{(o_t,a_t,o_{t+1})\sim\mu_{\mathrm{demo}}}
    \left[
      \bar w(o_t,a_t)
      \bigl\|
        \hat o_{t+1}(\phi)-o_{t+1}
      \bigr\|_2^2
    \right],
    \label{eq:vwi-loss-appx}
\end{equation}
where the unclipped signed coefficient is
\begin{equation}
    w(o_t,a_t)
    =
    \tilde d_{\phi}^{\pi_\theta}(o_t)\,
    \operatorname{sgn}\!\bigl(\hA(o_t,a_t)\bigr)
    \bigl|\hA(o_t,a_t)\bigr|^\alpha
    s(o_t),
    \label{eq:vwi-weight-appx}
\end{equation}
and $\bar w$ is the clipped, normalized, stop-gradient version used
by the implemented DRILL-VWI objective.
The core theoretical case
sets $\alpha=1$, $s(o_t)=1$, uses exact imagined-rollout visitation,
and removes clipping.  We denote this idealized loss by
$\lossvwi^{\mathrm{core}}(\phi)$.

To avoid overloading notation, we distinguish the state-visitation
marginal from the joint transition distributions.  Let
$\mu_{\mathrm{roll}}$ denote the time-averaged local imagined-rollout
transition distribution obtained by sampling a time index $t$ uniformly
from the $H$-step imagined rollout, then sampling
$o_t\sim\tilde d_{\phi}^{\pi_\theta}$,
$a_t\sim\pi_\theta(\cdot\mid o_{\le t})$, and
$o_{t+1}\sim\Wphi(\cdot\mid o_{\le t},a_t)$.  Let
$\mu_{\mathrm{demo}}$ denote the raw demonstration transition
distribution before applying VWI weights.  In demo-side expectations,
$\hA(o_t,a_t)$ denotes the stop-gradient advantage assigned to the
demo transition by the same rollout-derived advantage-transfer rule
used by the implemented VWI objective.  The mismatch between the local
rollout distribution and this raw demo-anchor distribution is kept as
the explicit residual $R_{\mathrm{demo}}$ below.

We work under the following assumptions.

\begin{assumption}[Outer smoothness]
\label{assump:A1-smooth}
$J:\Theta\to\mathbb R$ is $C^2$ in a neighborhood of $\theta$, with
bounded second derivatives, so that
\[
  J(\theta+\eta_\theta g)
  =
  J(\theta)
  +
  \eta_\theta\nabla_\theta J(\theta)^\top g
  +
  \mathcal O(\eta_\theta^2).
\]
\end{assumption}

\begin{assumption}[WM differentiability and Gaussian head]
\label{assump:A2-gauss}
For each fixed $(o_{\le t},a_t)$, the WM output distribution is a
Gaussian with mean $\hat o_{t+1}(\phi)$ and isotropic covariance
$\sigma_W^2I$, with $\phi\mapsto\hat o_{t+1}(\phi)$ continuously
differentiable.  Equivalently,
\begin{equation}
    \log \Wphi(o_{t+1}\mid o_{\le t},a_t)
    =
    -\frac{1}{2\sigma_W^2}
    \bigl\|
      \hat o_{t+1}(\phi)-o_{t+1}
    \bigr\|_2^2
    + \mathrm{const}.
    \label{eq:gauss-log-lik}
\end{equation}
\end{assumption}

\begin{assumption}[One-step on-policy stationarity and detached advantages]
\label{assump:A3-stationary}
Over one inner update in \eqref{eq:inner-update-appx}, the
imagined-rollout visitation marginal $\tilde d_{\phi}^{\pi_\theta}$
is treated as locally fixed.  The induced occupancy drift generally
carries horizon- and mixing-dependent
constants and is absorbed into the residual terms of the proposition.
We also treat the GRPO advantage as a stop-gradient with respect to
both $\theta$ and $\phi$ when forming the mixed derivative
$\nabla_\phi\nabla_\theta\hat J_{\mathrm{GRPO}}$.  In the vanilla
GRPO regime where the importance-ratio clip is inactive on the
expected step,
\[
  \nabla_\theta \hat J_{\mathrm{GRPO}}
  =
  \E_{o_t\sim\tilde d_{\phi}^{\pi_\theta},\,a_t\sim\pi_\theta(\cdot\mid o_{\le t})}
  \left[
    \hA(o_t,a_t)
    \nabla_\theta\log\pitheta(a_t\mid o_{\le t})
  \right].
\]
In practice, clip activation, advantage-estimation drift, and
occupancy drift are absorbed into the residual terms of the proposition;
we group the rollout-side local and multi-step effects into $R_H$ and
keep demo-anchored transfer effects in $R_{\mathrm{demo}}$.
\end{assumption}

\paragraph{Remark on Assumption~\ref{assump:A2-gauss}.}
The Gaussian-head assumption is used only to obtain the closed-form
identity in Eq.~\eqref{eq:score-as-mse-grad}.  For diffusion or
flow-based WMs, an analogous derivation can be written for the
corresponding score-matching or flow-matching objective, but the
simple pixel-MSE identity in Eq.~\eqref{eq:score-as-mse-grad} is
specific to the Gaussian-head presentation used in the main
statement.

\subsection{Main proposition}

\begin{proposition}[First-order expansion with compatibility, locality, and demo-anchor residuals;
restatement of Proposition~\ref{prop:v1-is-v2-first-order}]
\label{prop:v1-is-v2-first-order-app}
Under Assumptions~\ref{assump:A1-smooth}--\ref{assump:A3-stationary},
decompose the $H$-step trajectory-score contribution into a local
one-step term plus a residual $r_H$ with $\|r_H\|\le R_H$.  Define
the local policy-score contraction
\begin{equation}
  q_\theta(o_t,a_t)
  :=
  \left\langle
    \nabla_\theta\log\pitheta(a_t\mid o_{\le t}),
    \nabla_\theta J(\pitheta)
  \right\rangle .
  \label{eq:q-theta-app}
\end{equation}
Let $\mu_{\mathrm{roll}}$ denote the time-averaged local imagined-rollout
transition distribution induced by $\tilde d_{\phi}^{\pi_\theta}$, the
policy, and the WM one-step transition kernel.  Then the target
meta-gradient underlying DRILL-IMG admits the expansion
\begin{equation}
  \nabla_\phi J\bigl(\pi_{\theta^+(\phi)}\bigr)
  =
  \eta_\theta
  \E_{(o_t,a_t,o_{t+1})\sim\mu_{\mathrm{roll}}}
  \left[
    \hA(o_t,a_t) q_\theta(o_t,a_t)
    \nabla_\phi\log\Wphi(o_{t+1}\mid o_{\le t},a_t)
  \right]
  +
  \eta_\theta r_H
  +
  \mathcal O(\eta_\theta^2).
  \label{eq:policy-aware-score-app}
\end{equation}
Let $\mu_{\mathrm{demo}}$ denote the raw demonstration transition
distribution before applying VWI weights.  In the demo-side expectation
below, $\hA(o_t,a_t)$ denotes the stop-gradient advantage assigned by
the rollout-derived advantage-transfer rule used by implemented VWI.
For a scalar $\kappa>0$, define
\begin{align}
  \bigl(\nabla_\phi g(\phi,\theta)\bigr)^\top
  \nabla_\theta J(\theta)
  &=
  \E_{\mu_{\mathrm{roll}}}
  \left[
  \begin{aligned}
    &\hA(o_t,a_t)
    \left\langle
      \nabla_\theta\log\pitheta(a_t\mid o_{\le t}),
      \nabla_\theta J(\pitheta)
    \right\rangle \\
    &\qquad\cdot
    \nabla_\phi\log\Wphi(o_{t+1}\mid o_{\le t},a_t)
  \end{aligned}
  \right]
  +
  r_H
  +
  \mathcal O(\eta_\theta)
  \nonumber\\
  &=
  \E_{\mu_{\mathrm{roll}}}
  \left[
  \begin{aligned}
    &\hA(o_t,a_t)q_\theta(o_t,a_t) \\
    &\qquad\cdot
    \nabla_\phi\log\Wphi(o_{t+1}\mid o_{\le t},a_t)
  \end{aligned}
  \right]
  +
  r_H
  +
  \mathcal O(\eta_\theta).
  \label{eq:score-derivation-app}
\end{align}

Moreover, suppose that for the same scalar $\kappa>0$ the local
score-compatibility residual
\begin{equation}
  \epsilon_{\mathrm{comp}}
  :=
  \left(
  \E_{(o_t,a_t,o_{t+1})\sim\mu_{\mathrm{roll}}}
  \left[
    \hA(o_t,a_t)^2
    \bigl(q_\theta(o_t,a_t)-\kappa\bigr)^2
  \right]
  \right)^{1/2}
  \label{eq:compat-residual-app}
\end{equation}
is small, and define the demo-anchor residual
\begin{align}
  R_{\mathrm{demo}}
  :=
  \Biggl\|
  &\E_{\mu_{\mathrm{roll}}}
  \left[
  \begin{aligned}
    &\kappa\hA(o_t,a_t)\,
    \nabla_\phi
    \log\Wphi(o_{t+1}\mid o_{\le t},a_t)
  \end{aligned}
  \right]
  \nonumber\\
  &-
  \E_{\mu_{\mathrm{demo}}}
  \left[
  \begin{aligned}
    &\kappa \tilde d_{\phi}^{\pi_\theta}(o_t)\hA(o_t,a_t) \\
    &\qquad\cdot
    \nabla_\phi
    \log\Wphi(o_{t+1}\mid o_{\le t},a_t)
  \end{aligned}
  \right]
  \Biggr\| .
  \label{eq:demo-anchor-residual-app}
\end{align}
The positivity of $\kappa$ is a directional compatibility
assumption: a negative fitted scalar would indicate anti-alignment
rather than a useful VWI surrogate.  If the per-sample WM score is
bounded by
$\|\nabla_\phi\log\Wphi(o_{t+1}\mid o_{\le t},a_t)\|\le M_W$, then
under Assumption~\ref{assump:A2-gauss},
\begin{equation}
  \left\|
  \nabla_\phi J\bigl(\pi_{\theta^+(\phi)}\bigr)
  +
  \frac{\kappa\eta_\theta}{2\sigma_W^2}
  \nabla_\phi\lossvwi^{\mathrm{core}}(\phi)
  \right\|
  \le
  \eta_\theta M_W\epsilon_{\mathrm{comp}}
  +
  \eta_\theta R_H
  +
  \eta_\theta R_{\mathrm{demo}}
  +
  \mathcal O(\eta_\theta^2).
  \label{eq:vwi-compat-bound-app}
\end{equation}
\end{proposition}

\paragraph{Remark on $\alpha\ne1$.}
Because DRILL-VWI treats its weights as stop-gradient coefficients,
changing $\alpha$ does not introduce a $\nabla_\phi\hA$ term.  Instead,
it changes the scalar coefficient multiplying the per-sample MSE
gradient from the linear signed advantage $\hA$ to the nonlinear
signed power $\operatorname{sgn}(\hA)|\hA|^\alpha$.  The
score-compatible first-order surrogate is linear in the signed
advantage, so $\alpha=1$ preserves that local direction, whereas
$\alpha\ne1$ deliberately distorts it.  The $\alpha$-sweep in
Appendix~\ref{app:failure-modes} and Table~\ref{tab:coaching-ablation}
empirically studies this bias--variance trade-off.

\subsection{Proof}

\begin{proof}
The proof proceeds in four steps: (i) Taylor-expand the outer objective
around $\theta$ to first order in $\eta_\theta$; (ii) rewrite the
resulting $\nabla_\theta J$--$\nabla_\phi g$ product as a
policy-aware WM score with the local policy-score contraction kept
explicit; (iii) instantiate the score under the Gaussian-head
assumption; and (iv) compare the resulting direction with the core
DRILL-VWI surrogate.

\medskip
\noindent\textbf{Step 1: Linearization of the outer objective.}
Substituting Eq.~\eqref{eq:inner-update-appx} into
Eq.~\eqref{eq:outer-objective-appx} and applying
Assumption~\ref{assump:A1-smooth},
\begin{align}
  J\bigl(\pi_{\theta^+(\phi)}\bigr)
  &=
  J(\theta)
  +
  \eta_\theta
  \nabla_\theta J(\theta)^\top g(\phi,\theta)
  +
  \mathcal O(\eta_\theta^2).
  \label{eq:taylor-outer-app}
\end{align}
Differentiating with respect to $\phi$, and noting that $J(\theta)$
is independent of $\phi$,
\begin{align}
  \nabla_\phi J\bigl(\pi_{\theta^+(\phi)}\bigr)
  &=
  \eta_\theta
  \nabla_\phi
  \left[
    \nabla_\theta J(\theta)^\top g(\phi,\theta)
  \right]
  +
  \mathcal O(\eta_\theta^2)
  \nonumber\\
  &=
  \eta_\theta
  \bigl(\nabla_\phi g(\phi,\theta)\bigr)^\top
  \nabla_\theta J(\theta)
  +
  \mathcal O(\eta_\theta^2).
  \label{eq:outer-grad-linear-app}
\end{align}

\medskip
\noindent\textbf{Step 2: Policy-aware score identity with locality residual.}
We expand
\[
  \nabla_\phi g(\phi,\theta)
  =
  \nabla_\phi\nabla_\theta
  \hat J_{\mathrm{GRPO}}(\pitheta;\Wphi).
\]
In the vanilla GRPO regime,
\begin{equation}
  \hat J_{\mathrm{GRPO}}(\pitheta;\Wphi)
  =
  \E_{\tau=(o_0,a_0,\ldots,o_H)\sim P_{\Wphi}(\pitheta)}
  \left[
    \frac{1}{H}\sum_{t=0}^{H-1}
    \hA(o_t,a_t)\log\pitheta(a_t\mid o_{\le t})
  \right],
  \label{eq:grpo-surrogate-app}
\end{equation}
where $\hA$ is treated as a stop-gradient.  The factor $1/H$ makes
$\mu_{\mathrm{roll}}$ the uniform-time local transition distribution;
using an unnormalized sum would only multiply the one-step term by $H$,
which could be absorbed into the scalar compatibility constant or the
step size.  Taking $\nabla_\theta$ gives the policy-gradient score
$\hA\nabla_\theta\log\pitheta$ at the sampled time step.  The
derivative with respect to $\phi$ enters through the imagined rollout
distribution.

For an $H$-step imagined trajectory, the exact trajectory-score term
contains the full sum
$\nabla_\phi\log p_\phi(\tau)
=\sum_{j=0}^{H-1}
\nabla_\phi\log\Wphi(o_{j+1}\mid o_{\le j},a_j)$,
together with the dependence of later observations on earlier WM
DRILL-VWI keeps the local one-step transition score and
collects the remaining multi-step trajectory-score contribution into
$r_H$, with $\|r_H\|\le R_H$.  The additional mismatch from replacing
this rollout-side local score by the demo-anchored VWI objective is
captured by $R_{\mathrm{demo}}$ in
Eq.~\eqref{eq:demo-anchor-residual-app}.

Substituting Eq.~\eqref{eq:score-derivation-app} into
Eq.~\eqref{eq:outer-grad-linear-app} yields
Eq.~\eqref{eq:policy-aware-score-app}.
\medskip
\noindent\textbf{Step 3: Instantiate the score under Gaussian WM.}
Under Assumption~\ref{assump:A2-gauss}, differentiating
Eq.~\eqref{eq:gauss-log-lik} with respect to $\phi$ gives
\begin{equation}
  \nabla_\phi\log\Wphi(o_{t+1}\mid o_{\le t},a_t)
  =
  -\frac{1}{\sigma_W^2}
  \bigl(\hat o_{t+1}(\phi)-o_{t+1}\bigr)^\top
  \nabla_\phi\hat o_{t+1}(\phi)
  =
  -\frac{1}{2\sigma_W^2}
  \nabla_\phi
  \bigl\|
    \hat o_{t+1}(\phi)-o_{t+1}
  \bigr\|_2^2 .
  \label{eq:score-as-mse-grad}
\end{equation}
Thus, up to a negative scalar, the WM likelihood score equals the
per-sample MSE gradient.

\medskip
\noindent\textbf{Step 4: Compare with the core DRILL-VWI direction.}
Substituting Eq.~\eqref{eq:score-as-mse-grad} into
Eq.~\eqref{eq:policy-aware-score-app},
\begin{align}
  \nabla_\phi J\bigl(\pi_{\theta^+(\phi)}\bigr)
  &=
  -\frac{\eta_\theta}{2\sigma_W^2}
  \E_{(o_t,a_t,o_{t+1})\sim\mu_{\mathrm{roll}}}
  \left[
    \hA(o_t,a_t)q_\theta(o_t,a_t)
    \nabla_\phi
    \bigl\|
      \hat o_{t+1}(\phi)-o_{t+1}
    \bigr\|_2^2
  \right]
  \nonumber\\
  &\qquad
  +
  \eta_\theta r_H
  +
  \mathcal O(\eta_\theta^2).
  \label{eq:meta-grad-mse-form-app}
\end{align}
Add and subtract $\kappa$ inside the policy-score contraction:
\begin{align}
  \nabla_\phi J\bigl(\pi_{\theta^+(\phi)}\bigr)
  &=
  -\frac{\kappa\eta_\theta}{2\sigma_W^2}
  \E_{(o_t,a_t,o_{t+1})\sim\mu_{\mathrm{roll}}}
  \left[
    \hA(o_t,a_t)
    \nabla_\phi
    \bigl\|
      \hat o_{t+1}(\phi)-o_{t+1}
    \bigr\|_2^2
  \right]
  \nonumber\\
  &\quad
  -\frac{\eta_\theta}{2\sigma_W^2}
  \E_{(o_t,a_t,o_{t+1})\sim\mu_{\mathrm{roll}}}
  \left[
    \hA(o_t,a_t)
    \bigl(q_\theta(o_t,a_t)-\kappa\bigr)
    \nabla_\phi
    \bigl\|
      \hat o_{t+1}(\phi)-o_{t+1}
    \bigr\|_2^2
  \right]
  \nonumber\\
  &\quad
  +
  \eta_\theta r_H
  +
  \mathcal O(\eta_\theta^2).
  \label{eq:add-subtract-app}
\end{align}
The first term is the rollout-side version of the core VWI direction.
Replacing it by the demo-anchored core surrogate introduces exactly the
score-form mismatch $R_{\mathrm{demo}}$ in
Eq.~\eqref{eq:demo-anchor-residual-app}.  Equivalently, after adding
and subtracting the demo-anchored term,
\[
  -\frac{\kappa\eta_\theta}{2\sigma_W^2}
  \nabla_\phi\lossvwi^{\mathrm{core}}(\phi)
\]
appears, plus an additional residual bounded by
$\eta_\theta R_{\mathrm{demo}}$.  Using the score form of the
compatibility term and the bound
$\|\nabla_\phi\log\Wphi\|\le M_W$, Cauchy--Schwarz gives
\[
  \left\|
  \eta_\theta
  \E
  \left[
    \hA(o_t,a_t)
    \bigl(q_\theta(o_t,a_t)-\kappa\bigr)
    \nabla_\phi\log\Wphi(o_{t+1}\mid o_{\le t},a_t)
  \right]
  \right\|
  \le
  \eta_\theta M_W\epsilon_{\mathrm{comp}}.
\]
Combining the displays above proves Eq.~\eqref{eq:vwi-compat-bound-app}.
In the idealized case $q_\theta(o_t,a_t)=\kappa$, $R_H=0$, and
$R_{\mathrm{demo}}=0$, the residual terms vanish and gradient ascent on
$J$ is, to first order, equivalent to gradient descent on the core
signed DRILL-VWI surrogate up to the scalar
$\kappa\eta_\theta/(2\sigma_W^2)$.  In the general case, the deviation
is controlled by the compatibility residual, the trajectory-locality
residual, the demo-anchor residual, and the Taylor remainder.
This proves the proposition.
\end{proof}

\subsection{Corollaries}

The proposition yields three consequences.

\begin{corollary}[DRILL-VWI under score compatibility]
\label{cor:v1-alpha1-app}
If $q_\theta(o_t,a_t)=\kappa$ and $R_{\mathrm{demo}}=0$, then the
core DRILL-VWI direction at $\alpha=1$ matches the local one-step
component of the target first-order direction underlying DRILL-IMG up
to the Gaussian-head scalar $-\kappa\eta_\theta/(2\sigma_W^2)$.  For
the nonlocal target direction underlying DRILL-IMG, the remaining
mismatch is given by the trajectory-locality residual $r_H$ and the higher-order Taylor
remainder.  If the compatibility condition or demo-anchor matching is
imperfect, the additional mismatch is bounded by
Eqs.~\eqref{eq:compat-residual-app} and
\eqref{eq:demo-anchor-residual-app}.
\end{corollary}

\begin{corollary}[Stability is a variance-control multiplier]
\label{cor:stability-variance-app}
The positive stability factor $s(o_t)$ is not part of the
score-compatibility identity.  It rescales the signed local surrogate
toward states on which the WM heads agree, reducing variance without
flipping the policy-aware sign.  The ablation in
Table~\ref{tab:coaching-ablation} measures this variance-control
contribution.
\end{corollary}

\begin{corollary}[Sources of the DRILL-IMG--DRILL-VWI gap]
\label{cor:v2-minus-v1-app}
The empirical DRILL-IMG--DRILL-VWI gap has four theoretical sources:
the compatibility residual $\epsilon_{\mathrm{comp}}$, the
trajectory-locality residual $R_H$, the demo-anchor residual
$R_{\mathrm{demo}}$, and the higher-order Taylor remainder.  If one additionally assumes
$R_H=\mathcal O(\eta_\theta H^2)$, then the horizon-dependent part
of the VWI bias becomes $\mathcal O(\eta_\theta^2H^2)$.  Without this
additional locality assumption, the horizon sweep should be read as
evidence that the omitted trajectory-locality contribution grows with
$H$, rather than as a universal $H^2$ law.
\end{corollary}

\subsection{Discussion}

\paragraph{DRILL-VWI is a principled surrogate, not an exact truncation.}
Proposition~\ref{prop:v1-is-v2-first-order-app} places DRILL-VWI on
a principled footing: it is the closed-form, score-compatible local
surrogate obtained from the target first-order direction underlying DRILL-IMG.
Its accuracy depends on the compatibility residual
$\epsilon_{\mathrm{comp}}$, the trajectory-locality residual $R_H$,
the demo-anchor residual $R_{\mathrm{demo}}$, and the higher-order
Taylor remainder.  This is weaker than an exact
Taylor-truncation claim, but it is the assumption actually used by
the algorithm and is consistent with the diagnostic and horizon
ablations in Section~\ref{sec:experiments}.

\paragraph{Relation to VAML and value-equivalent models.}
The result is closest in spirit to value-aware model learning and
value-equivalent world models
\citep{farahmand2017value,grimm2020value}: the model is updated for
the downstream value of the policy-gradient signal it induces, not
for pixel fidelity alone.  DRILL-VWI differs in using the online
group-relative GRPO advantage as the local value surrogate, which
avoids fitting a separate value function for foundation-scale VLA
policies.

\paragraph{Where DRILL-IMG earns its complexity.}
DRILL-IMG is most useful when the omitted trajectory-locality terms
are large relative to the compatibility residual and stochastic HVP
noise.  This is why the horizon sweep in
Section~\ref{sec:exp-v1-v2-gap} is evidence for an increasing
multi-step locality residual, rather than a direct isolation of a
pure $\mathcal O(\eta_\theta^2H^2)$ term.

%
%
%

\section{Proof of Theorem~\ref{thm:v2-convergence}: DRILL-IMG convergence rate}
\label{app:thm-convergence}

This appendix proves Theorem~\ref{thm:v2-convergence} from
Section~\ref{sec:analysis-conv}, with explicit constants in place of the
$\mathcal{O}(\cdot)$ used in the body.  The argument adapts the biased-stochastic-gradient framework of
Ghadimi and Lan~\citep{ghadimi2013stochastic} to the truncated-HVP setting,
using only the three assumptions A4--A6 collected in
Appendix~\ref{app:notation-assumptions}.  We first prove the
fixed-policy conditional-objective guarantee used in the body, then
record the truncation-depth corollary and the biased-SGD interpretation
of the DRILL-VWI residual.

\subsection{Setup}
\label{app:thm-setup}

Fix the current policy parameters \(\bar\theta\) and define the
conditional outer objective
\[
J_{\bar\theta}(\phi)
:=
J\!\left(
\pi_{\bar\theta+\eta_\theta
\nabla_\theta\hJ_{\mathrm{GRPO}}(\pi_{\bar\theta};\Wphi)}
\right).
\]
For readability, we write \(J(\phi)\) for \(J_{\bar\theta}(\phi)\)
throughout the proof.  The idealized outer update analyzed below is
\begin{equation}
  \phi_{t+1}
  =
  \phi_t+\eta_\phi \hat g_t .
  \label{eq:img-update-appx}
\end{equation}
The estimator $\hat g_t$ is built from $B_{\mathrm{real}}$
simulator rollouts (for $\nabla_{\theta^{+}} J$) plus a $K$-step
truncated HVP
through the imagined-rollout dynamics (for the mixed second
derivative $\nabla_\phi\nabla_\theta\hJ_{\mathrm{GRPO}}$); see
Section~\ref{sec:framework-v2} and Algorithm~\ref{alg:img}.  Decompose
\begin{equation}
  \hat g_t
  \;=\;
  \nabla_\phi J(\phi_t) \;+\; b_t \;+\; n_t,
  \qquad
  b_t \;:=\; \E[\hat g_t \mid \phi_t] - \nabla_\phi J(\phi_t),
  \quad
  n_t \;:=\; \hat g_t - \E[\hat g_t \mid \phi_t],
  \label{eq:bias-noise-decomp}
\end{equation}
into its true gradient, a (deterministic, given $\phi_t$) bias
term~$b_t$, and a zero-mean noise term~$n_t$.  Assumptions A4 and
A6 give the bounds
\begin{equation}
  \E\bigl[\lVert n_t\rVert^2 \mid \phi_t\bigr] \;\le\; \sigma_g^2,
  \qquad
  \lVert b_t\rVert \;\le\; B\rho^K,
  \label{eq:noise-bias-bounds}
\end{equation}
the first uniformly in $t$, the second uniformly in $t$ and
independent of any random draw of $n_t$.  Assumption A5 gives
$L$-smoothness of $J$ in $\phi$.

\subsection{Theorem statement (with explicit constants)}

\begin{customthm}[DRILL-IMG convergence rate, explicit-constant form;
restatement of Theorem~\ref{thm:v2-convergence}]
\label{thm:v2-convergence-explicit}
Let $J_0:=J(\phi_0)$ and $\Delta := J^{*} - J_0$, where
$J^{*} := \sup_\phi J(\phi)$ is finite.  Under Assumptions A4--A6 and the iteration
\eqref{eq:img-update-appx} with constant step size
\begin{equation}
  \eta_\phi
    \;=\;
    \min\!\left(\frac{1}{4L},
                \;\sqrt{\frac{2\Delta}{LT\sigma_g^2}}\right),
  \label{eq:eta-phi-choice}
\end{equation}
the idealized stochastic-ascent iterates on this fixed conditional objective satisfy
\begin{equation}
  \min_{0\le t < T}
  \E\bigl\lVert \nabla_\phi J(\phi_t)\bigr\rVert^2
  \;\le\;
  \frac{16\,L\,\Delta}{T}
  \;+\;
  4\sqrt{\frac{2\,L\,\Delta\,\sigma_g^2}{T}}
  \;+\;
  3\,B^2\rho^{2K}.
  \label{eq:thm-rate-explicit}
\end{equation}
\end{customthm}

Setting $L\Delta$, $\sigma_g$, $B$, and $\rho$ aside as constants,
\eqref{eq:thm-rate-explicit} reduces to the rate stated in
\eqref{eq:convergence-rate} in the body.  If $\sigma_g=0$, one may
read the step-size choice as $\eta_\phi=1/(4L)$ and the stochastic
term vanishes.

\subsection{Proof}

The proof has four steps: (i)~a bias-aware ascent inequality from
$L$-smoothness, (ii)~control of the second-moment term using A4 and
A6, (iii)~telescoping over $T$ outer iterations, and (iv)~optimization
of $\eta_\phi$ subject to the smoothness step-size cap.

\medskip
\noindent\textbf{Step 1: bias-aware ascent inequality.}
$L$-smoothness of $J$ in $\phi$ (A5) gives the standard quadratic
lower bound for ascent:
\begin{equation*}
  J(\phi_{t+1})
  \;\ge\;
  J(\phi_t) \;+\;
  \bigl\langle \nabla_\phi J(\phi_t),\,
              \phi_{t+1} - \phi_t\bigr\rangle
  \;-\;
  \tfrac{L}{2}\,\lVert\phi_{t+1}-\phi_t\rVert^2.
\end{equation*}
Substituting~\eqref{eq:img-update-appx} and taking conditional
expectation given $\phi_t$,
\begin{equation}
  \E\!\bigl[J(\phi_{t+1}) \mid \phi_t\bigr]
  \;-\; J(\phi_t)
  \;\ge\;
  \eta_\phi\,
  \bigl\langle\nabla_\phi J(\phi_t),\,
              \E[\hat g_t \mid \phi_t]\bigr\rangle
  \;-\;
  \tfrac{L\eta_\phi^2}{2}\,
  \E\!\bigl[\lVert\hat g_t\rVert^2 \mid \phi_t\bigr].
  \label{eq:descent}
\end{equation}
Decomposing $\E[\hat g_t \mid \phi_t]
= \nabla_\phi J(\phi_t) + b_t$ from~\eqref{eq:bias-noise-decomp} and
applying Young's inequality
$\lvert\langle x,y\rangle\rvert \le \tfrac{1}{2}\lVert x\rVert^2
+ \tfrac{1}{2}\lVert y\rVert^2$ to the cross term,
\begin{equation}
  \bigl\langle\nabla_\phi J(\phi_t),\,
              \E[\hat g_t \mid \phi_t]\bigr\rangle
  \;=\;
  \lVert\nabla_\phi J(\phi_t)\rVert^2
  \;+\;
  \langle\nabla_\phi J(\phi_t),\, b_t\rangle
  \;\ge\;
  \tfrac{1}{2}\lVert\nabla_\phi J(\phi_t)\rVert^2
  \;-\;
  \tfrac{1}{2}\lVert b_t\rVert^2.
  \label{eq:cross-young}
\end{equation}

\medskip
\noindent\textbf{Step 2: second-moment control.}
Decomposing $\hat g_t = \E[\hat g_t \mid \phi_t] + n_t$ in the
second-moment term gives
\begin{equation*}
  \E\!\bigl[\lVert\hat g_t\rVert^2 \mid \phi_t\bigr]
  \;=\;
  \lVert\E[\hat g_t \mid \phi_t]\rVert^2
  \;+\;
  \E\!\bigl[\lVert n_t\rVert^2 \mid \phi_t\bigr].
\end{equation*}
For the mean-square term, applying $\lVert x+y\rVert^2 \le 2\lVert x\rVert^2 + 2\lVert y\rVert^2$
to $\E[\hat g_t \mid \phi_t] = \nabla_\phi J(\phi_t) + b_t$ and the
noise bound from~\eqref{eq:noise-bias-bounds}:
\begin{equation}
  \E\!\bigl[\lVert\hat g_t\rVert^2 \mid \phi_t\bigr]
  \;\le\;
  2\lVert\nabla_\phi J(\phi_t)\rVert^2
  \;+\;
  2\lVert b_t\rVert^2
  \;+\;
  \sigma_g^2.
  \label{eq:second-moment}
\end{equation}
Substituting~\eqref{eq:cross-young} and~\eqref{eq:second-moment}
into~\eqref{eq:descent}, and using $\lVert b_t\rVert^2\le B^2\rho^{2K}$,
\begin{align}
  \E\!\bigl[J(\phi_{t+1})-J(\phi_t) \mid \phi_t\bigr]
    &\;\ge\;
    \bigl(\tfrac{\eta_\phi}{2} - L\eta_\phi^2\bigr)\,
    \lVert\nabla_\phi J(\phi_t)\rVert^2
    \nonumber\\
    &\quad
    -\;
    \bigl(\tfrac{\eta_\phi}{2} + L\eta_\phi^2\bigr)\,
    B^2\rho^{2K}
    \;-\;
    \tfrac{L\eta_\phi^2}{2}\,\sigma_g^2.
  \label{eq:per-step-clean}
\end{align}
The step-size cap $\eta_\phi \le 1/(4L)$ from~\eqref{eq:eta-phi-choice}
makes the coefficient on $\lVert\nabla_\phi J\rVert^2$ at least
$\eta_\phi/4$ and the coefficient on $B^2\rho^{2K}$ at most
$3\eta_\phi/4$:
\begin{equation*}
  \tfrac{\eta_\phi}{2} - L\eta_\phi^2
    \;\ge\; \tfrac{\eta_\phi}{2} - \tfrac{\eta_\phi}{4}
    \;=\; \tfrac{\eta_\phi}{4},
  \qquad
  \tfrac{\eta_\phi}{2} + L\eta_\phi^2
    \;\le\; \tfrac{\eta_\phi}{2} + \tfrac{\eta_\phi}{4}
    \;=\; \tfrac{3\eta_\phi}{4}.
\end{equation*}
Plugging in and rearranging,
\begin{equation}
  \tfrac{\eta_\phi}{4}\,
  \lVert\nabla_\phi J(\phi_t)\rVert^2
  \;\le\;
  \E\!\bigl[J(\phi_{t+1})-J(\phi_t)\mid \phi_t\bigr]
  \;+\;
  \tfrac{3\eta_\phi}{4}\,B^2\rho^{2K}
  \;+\;
  \tfrac{L\eta_\phi^2}{2}\,\sigma_g^2.
  \label{eq:one-step-bound}
\end{equation}

\medskip
\noindent\textbf{Step 3: telescoping.}
Take the unconditional expectation of~\eqref{eq:one-step-bound} and
sum over $t = 0, 1, \dots, T-1$.  The telescoping right-hand side
collapses, and using $\E[J(\phi_T)] - J_0 \le J^{*} - J_0 = \Delta$,
\begin{equation*}
  \tfrac{\eta_\phi}{4}\sum_{t=0}^{T-1}
    \E\!\bigl[\lVert\nabla_\phi J(\phi_t)\rVert^2\bigr]
  \;\le\;
  \Delta
  \;+\;
  \tfrac{3T\eta_\phi}{4}\,B^2\rho^{2K}
  \;+\;
  \tfrac{TL\eta_\phi^2}{2}\,\sigma_g^2.
\end{equation*}
Dividing both sides by $T\eta_\phi/4$ and bounding the average from
below by the minimum,
\begin{equation}
  \min_{0\le t < T}
  \E\!\bigl[\lVert\nabla_\phi J(\phi_t)\rVert^2\bigr]
  \;\le\;
  \frac{4\Delta}{T\eta_\phi}
  \;+\;
  3\,B^2\rho^{2K}
  \;+\;
  2L\eta_\phi\,\sigma_g^2.
  \label{eq:rate-pre-eta}
\end{equation}
The middle term is the truncation-bias contribution and is independent
of $\eta_\phi$; the first and last terms are traded off by the step size.

\medskip
\noindent\textbf{Step 4: optimal step size.}
The function
$f(\eta_\phi) := 4\Delta/(T\eta_\phi) + 2L\eta_\phi\sigma_g^2$ is
convex and minimized over $\eta_\phi > 0$ at
$\eta_\phi^{\star} = \sqrt{2\Delta/(LT\sigma_g^2)}$, giving
$f(\eta_\phi^{\star}) = 4\sqrt{2L\Delta\sigma_g^2/T}$.  Two regimes:
\begin{itemize}[leftmargin=*, itemsep=2pt, topsep=2pt]
  \item \textbf{Stochastic regime} ($T$ large; $\eta_\phi^{\star}\le 1/(4L)$).
        Take $\eta_\phi = \eta_\phi^{\star}$.  Then $f(\eta_\phi)
        = 4\sqrt{2L\Delta\sigma_g^2/T}$, and \eqref{eq:rate-pre-eta}
        becomes
        $\min_t \E[\lVert\nabla J\rVert^2]
        \le 4\sqrt{2L\Delta\sigma_g^2/T} + 3B^2\rho^{2K}$.
  \item \textbf{Deterministic regime} ($T$ small; $\eta_\phi^{\star}> 1/(4L)$).
        Take $\eta_\phi = 1/(4L)$.  Then
        $4\Delta/(T\eta_\phi) = 16L\Delta/T$ and
        $2L\eta_\phi\sigma_g^2 = \sigma_g^2/2$, so
        $f(\eta_\phi) = 16L\Delta/T + \sigma_g^2/2$.
        In this regime $\eta_\phi^{\star}>1/(4L)$ is (after squaring)
        equivalent to $32L\Delta > T\sigma_g^2$, which is sharper than
        $128L\Delta > T\sigma_g^2$ and the latter is in turn equivalent
        (again after squaring) to
        $\sigma_g^2/2 < 4\sqrt{2L\Delta\sigma_g^2/T}$.  Hence
        $\sigma_g^2/2$ is upper-bounded by $4\sqrt{2L\Delta\sigma_g^2/T}$
        in this regime, and the same expression upper-bounds $f$.
\end{itemize}
Combining the two regimes, the choice $\eta_\phi = \min\bigl(1/(4L),\,
\sqrt{2\Delta/(LT\sigma_g^2)}\bigr)$ from~\eqref{eq:eta-phi-choice}
yields
\begin{equation*}
  \min_{0\le t < T}
  \E\!\bigl[\lVert\nabla_\phi J(\phi_t)\rVert^2\bigr]
  \;\le\;
  \frac{16L\Delta}{T}
  \;+\;
  4\sqrt{\frac{2L\Delta\sigma_g^2}{T}}
  \;+\;
  3\,B^2\rho^{2K},
\end{equation*}
which is~\eqref{eq:thm-rate-explicit}.  Rolling the explicit
constants $\{16, 4\sqrt{2}, 3\}$ into $\mathcal{O}(\cdot)$ recovers
the body bound~\eqref{eq:convergence-rate}. \hfill$\square$

\subsection{Corollary and DRILL-VWI interpretation}

\begin{proof}[Proof of \Cref{cor:K-threshold}]
The truncation-bias term in~\eqref{eq:thm-rate-explicit} is
$3 B^2\rho^{2K}$.  For any target $0<\epsilon<3B^2$, setting this
term below $\epsilon$ requires $\rho^{2K}\le \epsilon/(3B^2)$, i.e.,
$K \ge \tfrac{1}{2}\log\bigl(\epsilon/(3B^2)\bigr) / \log\rho$
(both numerator and denominator are negative for $\rho\in(0,1)$ and
$0<\epsilon < 3B^2$). If $\epsilon\ge 3B^2$, the condition is already
satisfied at $K=0$.  In particular the depth required to halve
the bias is $\Delta K = \log 2 / (2\log(1/\rho))$, so $K$ scales
logarithmically in $1/\epsilon$.

\paragraph{Why we use $K{=}2$ rather than reading it off this bound.}
The corollary above is a scaling statement, not a quantitative
threshold.  Under the conservative empirical constants reported in
Appendix~\ref{app:assumption-constants} ($\rho\!\approx\!0.80$,
$B\!\approx\!0.80$, $L\!\approx\!0.40$,
$\sigma_g\!\approx\!0.20$, $\Delta\!=\!\mathcal{O}(1)$), the
stochastic term at $T\!=\!10^{4}$ is
$4\sqrt{2L\Delta\sigma_g^2/T}\approx 7\!\times\!10^{-3}$ while the
truncation term at $K\!=\!2$ is $3B^2\rho^{4}\approx 0.79$, so the
bound does not place the truncation term below the stochastic
floor at any $K\!\le\!10$.  This is the regime where the bound is
loose: Section~\ref{sec:exp-hvp} measures the cosine between the
$K$-truncated and exact HVP and finds it exceeds $0.7$ already at
$K\!=\!2$ on gridworld and Meta-World~MT10 (the same regime the
constants in Appendix~\ref{app:assumption-constants} are estimated
from), and the BPTT-depth sweep in
Appendix~\ref{app:sensitivity-hparams} shows downstream SR
plateaus at $K\!=\!2$ ($K\!=\!3$ adds $\sim\!1.5\!\times$
wall-clock cost for $\le\!1$ pt SR change).  We therefore use
$K\!=\!2$ as an empirical return-per-compute optimum, and read
\Cref{cor:K-threshold} as the qualitative statement that
truncation bias decays geometrically in $K$ with rate
$\rho^{2}$.
\end{proof}

\paragraph{DRILL-VWI as a biased-SGD surrogate.}
By Proposition~\ref{prop:v1-is-v2-first-order}, the core DRILL-VWI
direction differs from the target meta-gradient underlying DRILL-IMG,
on the same fixed-policy conditional subproblem, by the combined
deterministic bias
\[
  \delta_{\mathrm{VWI}}
  =
  \mathcal{O}(\eta_\theta\epsilon_{\mathrm{comp}})
  +
  \mathcal{O}(\eta_\theta R_H)
  +
  \mathcal{O}(\eta_\theta R_{\mathrm{demo}})
  +
  \mathcal{O}(\eta_\theta^2).
\]
Treating this quantity as the deterministic bias term in the
biased-SGD proof gives the same stationary-neighborhood conclusion
as Theorem~\ref{thm:v2-convergence-explicit}, with \(B\rho^K\)
replaced by \(\delta_{\mathrm{VWI}}\).  Equivalently, the
squared-gradient floor is \(\mathcal{O}(\delta_{\mathrm{VWI}}^2)\),
corresponding to a gradient-norm neighborhood of radius
\(\mathcal{O}(\delta_{\mathrm{VWI}})\).  This is a local
biased-SGD interpretation of the VWI surrogate, not a proof that the
coupled DRILL-VWI training dynamics globally converge.  It also does
not make VWI a \(K=0\) truncated-BPTT estimator: VWI does not execute
BPTT, but uses the Gaussian-head score identity to form a local
closed-form surrogate.


\section{Implementation details}
\label{app:implementation}

\subsection{DRILL-VWI full algorithm}
\label{app:vwi-algorithm}

\begin{algorithm}[!htbp]
\small
\caption{DRILL-VWI training loop with signed visitation-weighted imagination.}
\label{alg:vwi}
\begin{algorithmic}[1]
\Require Pretrained WM $\Wphi$, policy $\pitheta$, frozen reward $r_\psi$,
         demonstration buffer $\mathcal D_{\mathrm{demo}}$, rollout horizon $H$,
         signed-advantage exponent $\alpha$, learning rates $\eta_\theta,\eta_\phi$.
\For{outer iteration $t = 1, 2, \dots, T$}
  \State Sample demonstration anchors
         $(o_{\le t},a_t,o_{t+1})\sim\mathcal D_{\mathrm{demo}}$
  \State Roll out imagined continuations $\tau_{\mathrm{img}}$ of length $H$
         in $\Wphi$ under $\pitheta$, starting from the anchor states
  \State Compute group-relative GRPO advantages $\hA$ on the imagined
         continuations and attach the resulting local coefficients to the
         corresponding demonstration anchors
  \State Estimate $\widehat{\tilde d}_{\phi}^{\pi_\theta}(o_t)$ by $k$-NN density in the WM
         latent buffer and compute confidence weight $s(o_t)>0$
  \State $\theta^{+} \gets \theta + \eta_\theta\nabla_\theta
         \hJ_{\mathrm{GRPO}}(\pitheta;\tau_{\mathrm{img}})$
  \State $w(o_t,a_t) \gets
         \widehat{\tilde d}_{\phi}^{\pi_\theta}(o_t)
         \operatorname{sgn}(\hA(o_t,a_t))
         |\hA(o_t,a_t)|^\alpha s(o_t)$
  \State $\bar w(o_t,a_t) \gets \operatorname{stopgrad}
         (\operatorname{normalize}_{\mathcal B}
         (\operatorname{clip}(w(o_t,a_t),-w_{\max},w_{\max})))$
  \State $\phi \gets \phi - \eta_\phi
         \nabla_\phi \lossvwi(\phi;\bar w)$
  \State $\theta \gets \theta^{+}$
\EndFor
\end{algorithmic}
\end{algorithm}

The signed coefficient in Algorithm~\ref{alg:vwi} is stopped before the
world-model update.  Thus the outer step differentiates only the
world-model prediction error, while the policy-aware visitation,
advantage sign, and confidence terms choose the local direction and
magnitude of the update.  Negative coefficients are allowed before
clipping; this is what lets DRILL-VWI locally move the WM score away
from transitions that the imagined policy-gradient signal judges harmful.

\FloatBarrier

\subsection{DRILL-IMG full algorithm}
\label{app:v2-algorithm}

\begin{algorithm}[!htbp]
\small
\caption{DRILL-IMG training loop with truncated HVP.}
\label{alg:img}
\begin{algorithmic}[1]
\Require Pretrained WM $\Wphi$, policy $\pitheta$, frozen reward $r_\psi$,
         target simulator $P_{\mathrm{env}}$, horizon $H$, BPTT depth $K$,
         outer batch $B_{\mathrm{real}}$, learning rates $\eta_\theta, \eta_\phi$.
\For{outer iteration $t = 1, 2, \dots, T$}
  \State \textbf{// Inner: imagine + policy update}
  \State $\tau_{\mathrm{img}} \gets \textsc{Rollout}(\pitheta, \Wphi, r_\psi; H)$
  \Comment{retain graph for $K$ steps}
  \State $\theta^{+} \gets \theta + \eta_\theta \nabla_\theta
         \hJ_{\mathrm{GRPO}}(\pitheta; \tau_{\mathrm{img}})$
  \State \textbf{// Outer: simulator rollouts}
  \State $\tau_{\mathrm{real}} \gets
         \{\textsc{Rollout}(\pi_{\theta^{+}}, P_{\mathrm{env}})\}_{i=1}^{B_{\mathrm{real}}}$
  \State $G \gets \hat\nabla_{\theta^{+}} J(\pi_{\theta^{+}}; \tau_{\mathrm{real}})$
  \Comment{standard policy-gradient estimator; $|\theta|$-vector}
  \State \textbf{// Mixed-Hessian-vector product through WM, truncated at depth $K$}
  \State $\hat g_\phi \gets \eta_\theta \cdot
         \nabla_\phi \!\bigl[\, G^{\top} \nabla_\theta
            \hJ_{\mathrm{GRPO}}(\pitheta; \tau_{\mathrm{img}}) \,\bigr]_{\mathrm{trunc}(K)}$
  \Comment{VJP through $K$ retained inner steps; $|\phi|$-vector}
  \State \textbf{// Outer WM update}
  \State $\phi \gets \phi + \eta_\phi\, \hat g_\phi$
  \State $\theta \gets \theta^{+}$
\EndFor
\end{algorithmic}
\end{algorithm}

\para{HVP implementation.}
The mixed-Hessian-vector product on line~9 is computed as a vector-Jacobian
product (VJP): we form the scalar
$s := G^{\top}\nabla_\theta \hJ_{\mathrm{GRPO}}(\pitheta;\tau_{\mathrm{img}})$
with $G$ detached, and call
\texttt{torch.autograd.grad(s, phi)} so that autograd backpropagates
through the $K$ retained inner steps.  This is dimensionally
unambiguous (the scalar gradient is automatically a $|\phi|$-vector)
and avoids materializing the $|\phi|\!\times\!|\theta|$ mixed Hessian.
The imagined rollout is unrolled for $K$ steps with
\texttt{create\_graph=True} and the remaining $H-K$ steps are unrolled
under \texttt{torch.no\_grad()}, which bounds peak activation memory
at $K\!\cdot\!|\Wphi|$ rather than $H\!\cdot\!|\Wphi|$.  At our default
$K\!=\!2$, $H\!=\!50$, and 1.2B-parameter WM this saves $\sim\!25\times$
memory relative to a full BPTT through the rollout.

\para{Numerical-stability controls.}
Three additional measures stabilize training of the
$>$1B-parameter WM under~\eqref{eq:img-update-appx}.  (i)~The outer
gradient $\hat g_\phi$ is clipped at
$\lVert\hat g_\phi\rVert\le 10$, which retains direction while
removing the heavy-tailed magnitude spikes typical of HVP-based
estimators.  (ii)~The HVP matmul on line~9 is accumulated in
\texttt{float32} regardless of the inner-loop precision; we observed
silent divergence at \texttt{bfloat16} accumulation.  (iii)~We freeze
the WM's normalization statistics during the outer update: only the
weights of $\Wphi$ are touched by~\eqref{eq:img-update-appx}.

\FloatBarrier

\subsection{\PUR{} estimator}
\label{app:pur-estimator}

We estimate $\nabla_\phi J(\pithetap)$ in~\eqref{eq:pur} by
finite differences in a random subspace, deliberately avoiding any
reuse of the meta-gradient machinery of DRILL-IMG.  This way the
same \PUR{} procedure applies to any WM training objective
(WMPO, RLVR-World, DRILL-VWI, DRILL-IMG, and others) rather than
being tied to one of them.

\para{Random-projection finite differences.}
Sample $r$ random Gaussian directions $u_1, \dots, u_r$ in
$\mathbb{R}^{|\phi|}$, each with i.i.d.\ entries from
$\mathcal{N}(0,1)$ and rescaled to unit Frobenius norm.  For each
direction, compute the two-sided directional derivative
\begin{equation}
  \hat e_i
    \;=\;
    \frac{
      J\!\bigl(\pi_{\theta^{+}(\phi + \delta\,u_i)}\bigr)
      \,-\,
      J\!\bigl(\pi_{\theta^{+}(\phi - \delta\,u_i)}\bigr)
    }{2\delta},
  \label{eq:pur-fd}
\end{equation}
and reconstruct the random-subspace projected gradient
$\hat g_{\mathrm{proj}} := \frac{1}{r}\sum_{i=1}^{r} \hat e_i\, u_i$.
Because the columns of $U := [u_1, \dots, u_r]$ are unit-norm Gaussian,
$\E[U U^{\top}] = (r/|\phi|)\,I$, so $\hat g_{\mathrm{proj}}$ is an
unbiased estimator of $(r/|\phi|)\,\nabla_\phi J$ in expectation; the
constant cancels when $\hat g_{\mathrm{proj}}$ is plugged
into~\eqref{eq:pur} since the cosine in the definition of \PUR{}
is scale-invariant.

\para{Hyperparameters.}
We use $r\!=\!128$, $\delta$ chosen so that $\lVert\delta\, u_i\rVert
= 10^{-3}\lVert\phi\rVert_F$, and $B_{\mathrm{real}}\!=\!30$ simulator
rollouts per evaluation of $J$ in~\eqref{eq:pur-fd}.  This costs
$2r\,B_{\mathrm{real}} = 7680$ simulator rollouts per \PUR{} measurement;
\PUR{} is logged every $M\!=\!500$ outer iterations rather than
every step.

\para{Variance reduction.}
For each direction $u_i$, the two rollouts in~\eqref{eq:pur-fd} are
paired: same seed, same initial state, same reward-model
forward pass.  This common-random-numbers control reduces the
variance of $\hat e_i$ by $1$--$2$ orders of magnitude on the
gridworld validation environment described
in~Appendix~\ref{app:exp-hvp-numerics}.  Bias from the finite step $\delta$
is $\mathcal{O}(\delta^{2})$ by Taylor expansion of $J$ around $\phi$,
which at our chosen $\delta$ is below the variance floor of the
random-subspace estimator.

\para{Validation.}
On the gridworld and Meta-World MT10 setups of
Appendix~\ref{app:exp-hvp-numerics}, where $\nabla_\phi J$ is exactly
computable by closed-form autograd, the random-projection estimator
attains cosine similarity $0.83\!\pm\!0.04$ with the true gradient at
$r\!=\!128$, confirming that \PUR{} measurements at this $r$ are
not bottlenecked by the random-subspace projection.

\FloatBarrier

\subsection{Exact-vs-truncated HVP on low-dimensional environments}
\label{app:exp-hvp-numerics}

The HVP-truncation experiment of Section~\ref{sec:exp-hvp} uses two
reduced environments where the exact meta-gradient is analytically
available, so the truncated estimator can be compared against ground
truth.

\para{Environment 1: 2D gridworld.}
A $10\times 10$ grid with four cardinal actions and a single goal
cell.  The WM is a softmax head over the 100 cells with a one-hidden-layer
MLP front-end ($\sim$5k parameters).  The policy is a tabular
softmax with one logit per (cell, action) pair (400 parameters).
Rewards are a fixed sparse function of position.  At this scale,
the full inner loop ($H\!=\!20$) is unrolled with PyTorch autograd
\texttt{create\_graph=True} and the exact meta-gradient
$\nabla_\phi J\!\bigl(\pithetap\bigr)$ is computed by reverse-mode
differentiation through the entire unroll.

\para{Environment 2: Meta-World MT10.}
The 10 manipulation tasks of Meta-World MT10 with the original
proprioceptive observation space (no images), an MLP-WM with two
hidden layers of width 256 ($\sim\!10^{5}$ parameters) and a
matched-architecture MLP-policy.  The exact meta-gradient is again
tractable by full unrolled autograd at this scale.

\para{Procedure.}
For each environment and each truncation depth $K\in\{1,2,3,4,5\}$,
we measure cosine similarity between the exact meta-gradient and
the truncated-$K$ HVP estimator at 1000 randomly drawn states
$\times$ 10 random seeds.  Spectral radius $\rho$ of the inner-loop
Jacobian (the constant of A6) is estimated by the largest singular
value of the unrolled Jacobian computed via power iteration.
Figure~\ref{fig:hvp-cosine} reports the resulting cosine
similarity as a function of $K$.  The empirical contraction rate
is $\hat\rho = 0.64\pm 0.05$ on gridworld and
$\hat\rho = 0.73\pm 0.07$ on MT10 (mean $\pm$ std over 10 seeds),
so the implied truncation factor is below $0.29$ at $K{=}2$, consistent
with the geometric truncation-bias decay in Corollary~\ref{cor:K-threshold}.

\begin{figure}[!htbp]
\centering
\begin{tikzpicture}
\begin{axis}[
    ybar,
    width=0.8\linewidth,
    height=2in,
    xlabel={Truncation Depth $K$},
    ylabel={Cosine Similarity},
    ymin=0, ymax=1.0,
    grid=major,
    enlarge x limits=0.15,
    legend pos=south east,
    legend style={nodes={scale=0.8, transform shape}}
]
\addplot[fill=blue!30, error bars/.cd, y dir=both, y explicit]
  coordinates { (1, 0.53) +- (0, 0.05) (2, 0.74) +- (0, 0.04) (3, 0.80) +- (0, 0.04) (4, 0.82) +- (0, 0.04) (5, 0.84) +- (0, 0.03) };
\addplot[fill=red!30, error bars/.cd, y dir=both, y explicit]
  coordinates { (1, 0.46) +- (0, 0.06) (2, 0.71) +- (0, 0.05) (3, 0.77) +- (0, 0.05) (4, 0.81) +- (0, 0.04) (5, 0.83) +- (0, 0.04) };
\addplot[thick, red, sharp plot, update limits=false] coordinates { (0, 0.7) (6, 0.7) };
\legend{Gridworld, MT10, 0.7 Threshold}
\end{axis}
\end{tikzpicture}
\caption{Exact-vs-truncated meta-gradient cosine similarity on
low-dimensional environments where the exact meta-gradient is
analytically tractable.  Error bars show $\pm 1$ std over 10
seeds $\times$ 1000 states.  $K\!=\!2$ already exceeds the $0.7$
threshold on both environments (on MT10 just barely at
$0.71\!\pm\!0.05$, the lower end of which falls slightly below
threshold); $K\!\ge\!3$ improves cosine by $<\!0.07$ at linearly
more wall-clock cost, hence the choice of $K\!=\!2$ at scale.}
\label{fig:hvp-cosine}
\end{figure}
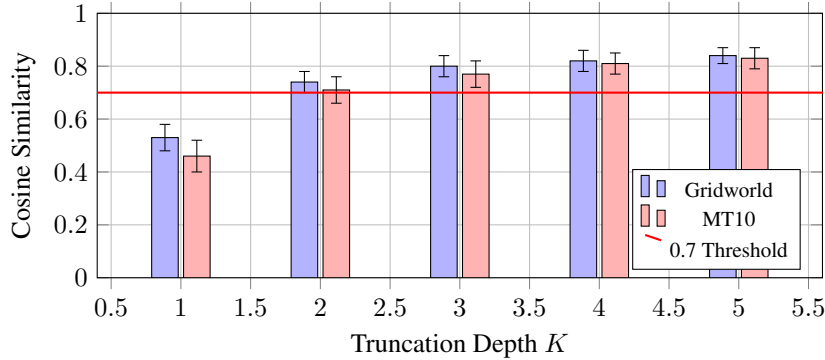

\FloatBarrier

\subsection{Task-relevant prediction score (TRPS)}
\label{app:trps}
The TRPS score used in Figure~\ref{fig:rollout-quality} is a
task-conditioned complement to pixel MSE that measures whether
the WM's imagined rollout preserves the information the
downstream reward head actually consumes.  Given a held-out demo
transition $(o_t,\ell,\{a_{t+k}\}_{k=0}^{H-1})$ with $H\!=\!10$,
we (i) roll out $H$ steps in the WM under $\phi$ to obtain the
imagined trajectory $\{\hat o_{t+k}(\phi)\}_{k=1}^{H}$,
(ii) roll out the same actions in the simulator to obtain the
ground-truth trajectory $\{o_{t+k}\}_{k=1}^{H}$,
(iii) pass both through the frozen VLM reward head $r_\psi$ to
obtain per-step progress scores $\{\hat r_k\}$ and $\{r_k\}$,
and (iv) discretize each score into a three-way label
$\{\text{regress},\text{neutral},\text{progress}\}$ by fixed
thresholds $(-0.1,\,+0.1)$ chosen on a held-out demo split.
TRPS is the per-trajectory fraction of the $H$ steps on which
$\hat r_k$ and $r_k$ receive the same label; the
Figure~\ref{fig:rollout-quality} curves report the mean TRPS
over $500$ held-out transitions per training checkpoint, averaged
over $5$ seeds.  Pixel MSE on the same transitions is measured
on the $1$-step prediction $\hat o_{t+1}(\phi)$ and is the same
quantity reported in Table~\ref{tab:main-headline}.  Per-seed
standard deviations at convergence are $\pm 0.020$ (WMPO) and
$\pm 0.028$ (DRILL-IMG).

\FloatBarrier

\subsection{Visitation estimator}
\label{app:visitation}

The factor $\widehat{\tilde d}_{\phi}^{\pi_\theta}(o_t)$ in the
DRILL-VWI weight~\eqref{eq:vwi-weight} estimates the model-rollout
visitation marginal $\tilde d_{\phi}^{\pi_\theta}(o_t)$ by
nearest-neighbour density in the world model's latent space, using a
streaming FAISS index.

\para{Latent representation.}
Observations $o_t$ are projected into the WM's pretrained encoder
output $z_t \in \mathbb{R}^{256}$.  The encoder is frozen; only its
forward pass is used.  Latents are L2-normalized before insertion
into the FAISS index, so cosine and Euclidean nearest neighbours
agree.

\para{Index structure.}
We maintain an online FAISS \texttt{IVF1024,Flat} index with capacity
$L = 10^{6}$ recent on-policy latents, refreshed in a circular buffer:
each new on-policy rollout overwrites the oldest $H$ entries.  IVF
clustering is rebuilt every 100 outer iterations to keep the inverted
lists balanced under policy drift; between rebuilds, new latents are
assigned to existing clusters by the 1-NN rule.

\para{Density estimate.}
Given a query latent $z_t$, let $N_k(z_t)$ denote its $k$-nearest
neighbours in the index ($k = 16$).  The visitation estimate is
\begin{equation*}
  \widehat{\tilde d}_{\phi}^{\pi_\theta}(o_t)
    \;\propto\;
    \frac{\sum_{z\in N_k(z_t)} \exp\!\bigl(-\lVert z - z_t\rVert^2 / 2h^2\bigr) \,+\, \lambda}{
          L \,+\, \lambda\,|Z|},
\end{equation*}
with bandwidth $h\!=\!0.1$ tuned on a held-out demonstration split,
Laplace-smoothing weight $\lambda\!=\!1$, and total observed-latent
count $|Z|$.  The right-hand side is normalized so that
$\sum_{o_t} \widehat{\tilde d}_{\phi}^{\pi_\theta}(o_t) = 1$ over
each minibatch before being used as the model-rollout visitation
multiplier $\widehat{\tilde d}_{\phi}^{\pi_\theta}(o_t)$
in~\eqref{eq:vwi-weight}.

\para{Refresh schedule.}
The index is incrementally extended every 10 inner iterations to
amortize the FAISS-add cost; the $k$-NN search itself is performed
in parallel across the minibatch using the FAISS GPU backend, with
a measured overhead of $<\!2\%$ relative to the inner-loop forward
pass.

\FloatBarrier

\subsection{Hyperparameters}
\label{app:hparams}

Table~\ref{tab:hparams} reports the training hyperparameters used
for the five-benchmark $\pi_0$-3B headline run and the three-benchmark
OpenVLA-7B replication of Section~\ref{sec:exp-headline}.  Values that
differ across benchmarks are listed per-benchmark; values shared across
all rows are summarized in the caption.

\begin{table}[!htbp]
\centering
\scriptsize
\setlength{\tabcolsep}{2.1pt}
\renewcommand{\arraystretch}{0.98}
\resizebox{0.98\textwidth}{!}{%
\begin{tabular}{@{}l l l c c c c c c c c@{}}
\toprule
\textbf{Env} & \textbf{VLA} & \textbf{WM}
 & $\eta_\theta$ & $\eta_\phi$ & batch & $H$ & $H_a$ & $K$ & $\alpha$ & seeds \\
\midrule
LIBERO-Long      & $\pi_0$-3B  & 1.2B Ctrl-World & $1\!\times\!10^{-5}$ & $5\!\times\!10^{-6}$ & 512 & 50 & 16 & 2 & 1 & 5 \\
CALVIN ABC$\to$D & $\pi_0$-3B  & 1.2B Ctrl-World & $1\!\times\!10^{-5}$ & $5\!\times\!10^{-6}$ & 512 & 50 & 16 & 2 & 1 & 5 \\
RoboCasa         & $\pi_0$-3B  & 1.2B Ctrl-World & $8\!\times\!10^{-6}$ & $4\!\times\!10^{-6}$ & 384 & 50 & 16 & 2 & 1 & 5 \\
Meta-World MT50  & $\pi_0$-3B  & 1.2B Ctrl-World & $1\!\times\!10^{-5}$ & $5\!\times\!10^{-6}$ & 512 & 50 & 16 & 2 & 1 & 5 \\
SimplerEnv       & $\pi_0$-3B  & 1.2B Ctrl-World & $1\!\times\!10^{-5}$ & $5\!\times\!10^{-6}$ & 512 & 50 & 16 & 2 & 1 & 5 \\
\midrule
LIBERO-Long      & OpenVLA-7B & 1.2B Ctrl-World & $5\!\times\!10^{-6}$ & $2\!\times\!10^{-6}$ & 256 & 50 & 16 & 2 & 1 & 5 \\
CALVIN ABC$\to$D & OpenVLA-7B & 1.2B Ctrl-World & $5\!\times\!10^{-6}$ & $2\!\times\!10^{-6}$ & 256 & 50 & 16 & 2 & 1 & 5 \\
RoboCasa         & OpenVLA-7B & 1.2B Ctrl-World & $4\!\times\!10^{-6}$ & $2\!\times\!10^{-6}$ & 192 & 50 & 16 & 2 & 1 & 5 \\
\bottomrule
\end{tabular}%
}
\begingroup
\setlength{\abovecaptionskip}{4pt}
\setlength{\belowcaptionskip}{0pt}
\caption{Training hyperparameters.  Values shared across all rows:
group size $G\!=\!8$ for GRPO; outer-loop optimizer AdamW with
$\beta_1\!=\!0.9$, $\beta_2\!=\!0.95$, weight decay $10^{-4}$; inner-loop
optimizer plain SGD with momentum $0.9$; gradient clipping at
$\lVert\hat g_\phi\rVert\le 10$ for outer and
$\lVert\nabla_\theta\hJ\rVert\le 1$ for inner; mixed-precision
\texttt{bfloat16} forward, \texttt{float32} HVP accumulation; 5
seeds for headline rows, 3 seeds for ablations.  Effective batch is
the per-GPU batch times 8 H20 GPUs (data-parallel across the
24-GPU pool, with a 16-way fragment used for the WM forward and
an 8-way fragment used for evaluation in parallel).}
\label{tab:hparams}
\endgroup
\end{table}

\FloatBarrier

\subsection{Simulator-budget-matched Sim-PPO baseline}
\label{app:simppo}

Sim-PPO in Table~\ref{tab:main-headline} is a model-free control
for the small simulator budget used by DRILL-IMG's outer
gradient.  ``Matched'' in this paper always means
\emph{simulator-episode-matched}: across the same number of outer
iterations, Sim-PPO and DRILL-IMG consume an identical number of
simulator episodes ($B_{\mathrm{real}}{=}30$ per outer step).  We
do not match on wall-clock time or on gradient updates, because
the simulator is the only training resource shared between the
two methods: DRILL-IMG additionally pays for imagined rollouts
under $\Wphi$ and for the truncated-HVP backward, but those costs
do not consume simulator budget and are therefore not what
Sim-PPO is controlling for.

Sim-PPO is not trained from scratch with only
$B_{\mathrm{real}}{=}30$ rollouts per update, which would be an
uninformative strawman.  Instead, Sim-PPO starts from the same
VLA checkpoint and the same offline-demonstration warm start as
WMPO, then receives exactly the simulator interactions that
DRILL-IMG would have consumed for its outer gradient.  For each
outer iteration we collect $B_{\mathrm{real}}{=}30$ simulator
episodes, add them to a replay buffer of the most recent $600$
episodes, and run PPO for 4 epochs with minibatch size 64, clip
ratio 0.2, value-loss coefficient 0.5, entropy coefficient 0.01,
and learning rate $2\times 10^{-6}$.  This gives Sim-PPO the best
plausible use of the matched simulator budget without access to
imagined rollouts; the low success rates in
Table~\ref{tab:main-headline} therefore measure the value of WM
amplification rather than a failure to tune PPO.

\FloatBarrier

\subsection{Compute budget}
\label{app:compute}

Table~\ref{tab:compute} reports the H20-hour breakdown across
training phases for the headline experiments.  All numbers are
measured on a single 24-GPU NVIDIA H20 cluster; phases that use
fewer than 24 GPUs ran in parallel with one another (the
evaluation phase, in particular, ran continuously alongside
training).  ``Inner-loop'' counts the imagined rollouts under
$\Wphi$ that drive \eqref{eq:wmpo-grpo}; ``Outer-loop'' counts
the simulator rollouts plus the truncated-HVP backward.  Numbers
are summed over the five $\pi_0$-3B headline benchmarks and the
three OpenVLA-7B replication benchmarks, each with five seeds;
baselines (WMPO, RLVR-World) reuse the same WM-pretrain phase, so
the marginal cost of running DRILL-IMG on top of an existing WMPO
pipeline is the ``Outer-loop'' row alone.

\begin{table}[!htbp]
\centering
\small
\setlength{\tabcolsep}{4pt} 
\renewcommand{\arraystretch}{1.0} 

\begin{tabular}{@{}lcccc@{}}
\toprule
\textbf{Phase} & GPUs & H20-hours & \% total & shared with baselines \\
\midrule
WM pretrain                      & 24 & 7{,}200  & 42\% & yes (WMPO/RLVR-World) \\
Inner loop (GRPO + imagined)     & 16 & 4{,}800  & 28\% & yes (WMPO/RLVR-World) \\
Outer loop (sim rollouts + HVP)  & 16 & 3{,}840  & 23\% & no  (DRILL-IMG-only) \\
Evaluation (reported tasks)      &  8 & 1{,}200  &  7\% & yes (all methods) \\
\midrule
\textbf{Total}                   & -- & \textbf{17{,}040} & 100\% & -- \\
\bottomrule
\end{tabular}

\begingroup
\setlength{\abovecaptionskip}{4pt}
\setlength{\belowcaptionskip}{0pt}
\caption{Compute budget breakdown for the reported DRILL-IMG runs
(five $\pi_0$-3B benchmarks plus three OpenVLA-7B replication
benchmarks, each with five seeds), measured on a 24-GPU NVIDIA H20
cluster over $\sim$4 weeks of wall-clock.
The marginal cost of running DRILL-IMG on top of an existing WMPO
pipeline is the ``Outer loop'' row alone ($\sim$3.8k H20-h),
since the WM-pretrain and inner-loop phases are shared with the
WMPO/RLVR-World baselines.}
\label{tab:compute}
\endgroup
\end{table}
%
%

\section{Reproducibility, Datasets, and Licenses}
\label{app:reproducibility}

\subsection{Reproducibility statement}
\label{app:reproducibility-statement}

We provide the following materials so that all empirical claims in
this paper can be reproduced by an independent group with access to
comparable compute (cf.~Appendix~\ref{app:compute}):
\begin{itemize}[leftmargin=*, itemsep=2pt, topsep=2pt]
  \item \textbf{Code.} All training, evaluation, and analysis code,
        including the truncated-HVP implementation
        (Appendix~\ref{app:v2-algorithm}), the FAISS-based
        visitation estimator (Appendix~\ref{app:visitation}), and
        the random-projection \PUR{} estimator
        (Appendix~\ref{app:pur-estimator}), will be released under
        an open-source license.
  \item \textbf{Models.} Trained checkpoints for the headline
        DRILL-IMG runs on each of the five benchmarks and both
        backbones will be released alongside the code.  Pretrained
        $\pi_0$-3B and OpenVLA-7B backbones are obtained from
        their respective public releases \citep{pi0_2024,openvla2024}; we do not redistribute them.
  \item \textbf{Hyperparameters.}  The full per-benchmark
        hyperparameter grid is in Appendix~\ref{app:hparams}
        (Table~\ref{tab:hparams}); shared training configuration is
        listed in the caption of that table.
  \item \textbf{Compute.} The headline experiments were run on a
        24-GPU NVIDIA H20 cluster over $\sim$4 weeks of wall-clock
        ($\sim$17k H20-hours total).  The per-phase breakdown is
        in Appendix~\ref{app:compute} (Table~\ref{tab:compute}).
        The marginal cost of DRILL-IMG on top of an existing WMPO
        pipeline is the ``Outer loop'' row alone
        ($\sim$3.8k H20-hours), since the WM-pretrain and
        inner-loop phases are shared with the WMPO/RLVR-World
        baselines.
  \item \textbf{Seeds.} Headline rows use 5 fixed seeds
        $\{0, 1, 2, 3, 4\}$; ablations use 3 seeds $\{0, 1, 2\}$.
        Reported numbers are means; statistical claims use Welch's
        $t$-test on the per-seed mean SR within each benchmark.
  \item \textbf{Human evaluation.}  The forced-choice human
        evaluation in Section~\ref{sec:exp-human-eval} uses six raters
        on $N{=}120$ paired imagined rollouts per row.  Raters were
        recruited internally; no personally identifiable
        information was collected, recorded, or released.  The
        rating interface and instructions are included in the
        released code.
\end{itemize}

\subsection{Datasets and benchmarks}
\label{app:datasets}

All five benchmarks are publicly released and are used unmodified
under the licenses stated by their respective authors.

\begin{itemize}[leftmargin=*, itemsep=2pt, topsep=2pt]
  \item \textbf{LIBERO}~\citep{libero2023}: lifelong-learning
        manipulation benchmark.  We use the LIBERO-Long suite (10
        long-horizon tasks) with the official $50$
        demonstrations per task.  License: MIT.
  \item \textbf{CALVIN}~\citep{calvin2022}: long-horizon
        language-conditioned manipulation.  We use the ABC$\to$D
        split for cross-environment generalization.  License: MIT.
  \item \textbf{RoboCasa}~\citep{robocasa2024}: large-scale
        kitchen-manipulation benchmark; we use the $\sim\!20$
        atomic-skill tasks.  License: MIT.
  \item \textbf{Meta-World} MT50: 50-task multi-task manipulation
        suite, used here in its proprioceptive-observation form
        for the exact-HVP audit (Section~\ref{sec:exp-hvp},
        Appendix~\ref{app:exp-hvp-numerics}).  License: MIT.
  \item \textbf{SimplerEnv}: a simulated-evaluation suite for
        VLA policies; we use the Bridge Data~V2 and Google~Robot
        subsets.  License: Apache~2.0.
\end{itemize}

\subsection{Pretrained-model licenses}
\label{app:model-licenses}

\begin{itemize}[leftmargin=*, itemsep=2pt, topsep=2pt]
  \item $\pi_0$-3B~\citep{pi0_2024}: redistributed by the
        original authors under their published license; our paper
        only consumes the public checkpoint and does not
        redistribute weights.
  \item OpenVLA-7B~\citep{openvla2024}: released under the
        MIT license.
  \item Ctrl-World 1.2B world model~\citep{ctrlworld2025}: used
        as the default WM backbone; license per the original
        release.
\end{itemize}

\subsection{Scope of evaluation}
\label{app:eval-scope}

All quantitative results in this paper are obtained in simulation
(LIBERO, CALVIN, RoboCasa, Meta-World, SimplerEnv); we report no
real-robot results.  Bridging the simulator-to-real gap is a
separate problem that DRILL does not address, and we make no
claim about real-world deployment.
%
%

\section{Sensitivity, Failure Modes, and Additional Ablations}
\label{app:sensitivity}

This appendix collects auxiliary experimental material moved out of the
main experiments section, together with sensitivity sweeps, failure modes,
and additional ablations.
The goal is to keep Section~\ref{sec:experiments}
focused on the main evidence while preserving all supporting checks.
\subsection{Auxiliary main-result checks}
\label{app:main-auxiliary}
\label{app:main-result-plots}

The main text keeps only a compact summary of the compute-matched result.
Table~\ref{tab:compute-matched} reports the full control: DRILL's gain is
not explained by extra WMPO updates or by ordinary convergence speed.

\begin{table}[!htbp]
\centering
\small
\setlength{\tabcolsep}{4pt}
\renewcommand{\arraystretch}{1.0}
\begin{tabular}{@{}lcc@{}}
\toprule
\textbf{Method} & Outer steps & Mean SR \\
\midrule
WMPO & 10k & $64.2\!\pm\!1.6$ \\
WMPO, 1.4$\times$ compute & 14k & $65.1\!\pm\!1.9$ \\
WMPO, 2.0$\times$ compute & 20k & $64.4\!\pm\!2.2$ \\
\textbf{DRILL-IMG} & 10k & $\mathbf{81.3\!\pm\!1.3}$ \\
\bottomrule
\end{tabular}
\begingroup
\setlength{\abovecaptionskip}{5pt}
\setlength{\belowcaptionskip}{0pt}
\caption{Compute-matched control on LIBERO-Long. Extending WMPO outer updates yields marginal, non-monotone changes, confirming that DRILL's gain is not merely compute-driven.}
\label{tab:compute-matched}
\endgroup
\end{table}

\FloatBarrier

\subsection{Objective ablations}
\label{app:objective-ablations}

Table~\ref{tab:objective-ablations} reports the LIBERO, CALVIN, and
RoboCasa ablations summarized in Section~\ref{sec:exp-ablations}.  The
single-factor ablations should be read as diagnostics of the signed
coaching surrogate rather than as standalone methods: visitation selects
where policy improvement is likely to matter, the signed advantage selects
the update direction, and the confidence term avoids amplifying
low-support latent neighborhoods.

\begin{table}[!htbp]
\centering
\small
\setlength{\tabcolsep}{2.6pt}
\renewcommand{\arraystretch}{0.96}
\begin{tabular}{@{}lccc@{}}
\toprule
\textbf{Variant} & Mean SR & MSE $\downarrow$ & \PUR{} $\downarrow$ \\
\midrule
WMPO & $58.2\!\pm\!1.7$ & $8.4\!\pm\!0.3$ & $1.42\!\pm\!0.08$ \\
\midrule
\multicolumn{4}{@{}l}{Coaching signal} \\
Visitation only & $64.2\!\pm\!1.5$ & $8.6\!\pm\!0.3$ & $1.18\!\pm\!0.07$ \\
Advantage only  & $66.8\!\pm\!1.4$ & $8.9\!\pm\!0.4$ & $1.06\!\pm\!0.06$ \\
Stability only  & $60.1\!\pm\!1.6$ & $8.5\!\pm\!0.3$ & $1.31\!\pm\!0.08$ \\
\midrule
\multicolumn{4}{@{}l}{Advantage exponent} \\
VWI ($\alpha\!=\!0$) & $65.0\!\pm\!1.5$ & $8.8\!\pm\!0.3$ & $1.14\!\pm\!0.06$ \\
VWI ($\alpha\!=\!0.5$) & $67.6\!\pm\!1.4$ & $9.0\!\pm\!0.4$ & $0.94\!\pm\!0.06$ \\
\textbf{VWI ($\alpha\!=\!1$)} & $\mathbf{69.2\!\pm\!1.3}$ & $9.3\!\pm\!0.3$ & $\mathbf{0.73\!\pm\!0.05}$ \\
VWI ($\alpha\!=\!2$) & $67.9\!\pm\!1.6$ & $9.5\!\pm\!0.4$ & $0.76\!\pm\!0.05$ \\
\midrule
\multicolumn{4}{@{}l}{Value-aware/full objective} \\
VA-WM & $62.4\!\pm\!2.0$ & $9.0\!\pm\!0.4$ & $1.10\!\pm\!0.07$ \\
\textbf{DRILL-IMG} & $\mathbf{73.4\!\pm\!1.5}$ & $9.8\!\pm\!0.4$ & $\mathbf{0.36\!\pm\!0.04}$ \\
\bottomrule
\end{tabular}
\begingroup
\setlength{\abovecaptionskip}{5pt}
\setlength{\belowcaptionskip}{0pt}
\caption{Objective ablations on LIBERO, CALVIN, and RoboCasa.  Values are
mean$\pm$std over three seeds; MSE is in $10^{-3}$.}
\label{tab:objective-ablations}
\label{tab:coaching-ablation}
\label{tab:va-wm-ablation}
\endgroup
\end{table}

\FloatBarrier

\subsection{Horizon sweep and short-horizon checks}
\label{app:horizon-sweep}

The LIBERO-Long horizon sweep in Table~\ref{tab:horizon-sweep} supports
the horizon-dependent criterion used in Section~\ref{sec:exp-v1-v2-gap}.
DRILL-IMG gains more over DRILL-VWI as $H$ increases from 20 to 200, but
both DRILL variants eventually suffer when imagined rollouts become noisy.
The short-horizon LIBERO-Long checks give gaps of 0.6 points at $H{=}10$
and $-0.2\!\pm\!1.1$ at $H{=}5$, while SimplerEnv gives the corresponding
short-horizon benchmark case.
\begin{table}[!htbp]
\centering
\small
\setlength{\tabcolsep}{2.6pt}
\renewcommand{\arraystretch}{0.98}
\begin{tabular}{@{}rccccc@{}}
\toprule
\textbf{$H$} & WMPO & VWI & \textbf{IMG}
             & IMG$-$WMPO & IMG$-$VWI \\
\midrule
20  & $67.4\!\pm\!1.5$ & $72.6\!\pm\!1.3$ & $74.4\!\pm\!1.5$          &  $7.0$ & $1.8$ \\
50  & $64.2\!\pm\!1.6$ & $76.1\!\pm\!1.2$ & $\mathbf{81.3\!\pm\!1.3}$ & $17.1$ & $5.2$ \\
100 & $58.0\!\pm\!2.2$ & $75.9\!\pm\!1.5$ & $83.4\!\pm\!1.7$          & $25.4$ & $7.5$ \\
200 & $49.8\!\pm\!2.8$ & $70.5\!\pm\!2.1$ & $78.6\!\pm\!2.5$          & $28.8$ & $8.1$ \\
\bottomrule
\end{tabular}
\begingroup
\setlength{\abovecaptionskip}{5pt}
\setlength{\belowcaptionskip}{0pt}
\caption{Horizon sweep on LIBERO-Long.  VWI and IMG denote DRILL-VWI and
DRILL-IMG; five seeds, with $H{=}50$ as the default.}
\label{tab:horizon-sweep}
\endgroup
\end{table}

\FloatBarrier

\subsection{Training diagnostics}
\label{app:training-diagnostics}
\label{app:pur-diagnostics}

The \PUR{} trajectories in Figure~\ref{fig:pur-curves} provide the
checkpoint-level view behind the end-of-training \PUR{} values in the
headline table.
WMPO drifts upward, RLVR-World stays near a plateau,
DRILL-VWI remains bounded, and DRILL-IMG moves steadily downward.
Figure~\ref{fig:rollout-quality} separates global pixel fidelity from
task-relevant rollout quality.
WMPO obtains lower one-step pixel MSE,
but DRILL-IMG obtains the higher TRPS score used in
Appendix~\ref{app:trps}, matching the success-rate gap in the main text.
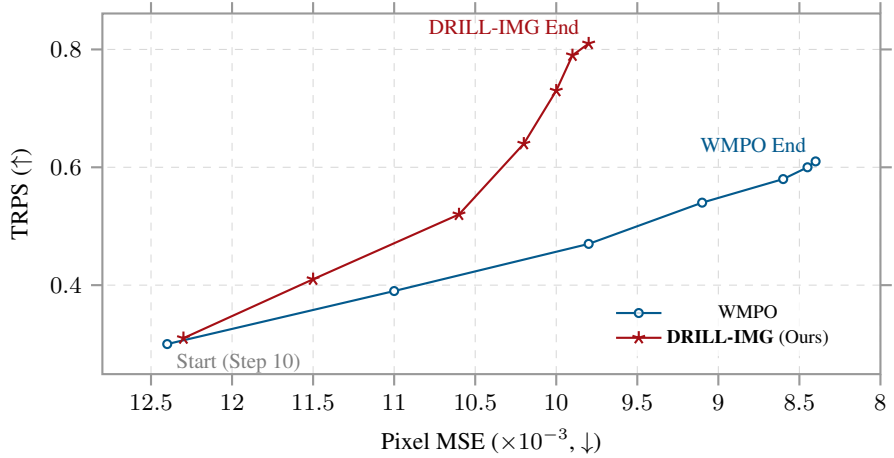
\begin{figure}[!htbp]
\centering
\begin{tikzpicture}
\begin{axis}[
    modern paper,
    width=0.85\linewidth,
    height=2.5in,
    xlabel={Pixel MSE ($\times 10^{-3}$, $\downarrow$)},
    ylabel={TRPS ($\uparrow$)},
    x dir=reverse, 
    legend style={at={(0.95,0.05)}, anchor=south east}
]

\addplot[thick, color=DeepNavy, mark=*, mark options={fill=white, solid}, mark size=1.5pt] coordinates {
  (12.4, 0.30) (11.0, 0.39) (9.8, 0.47) (9.1, 0.54) (8.6, 0.58) (8.45, 0.60) (8.4, 0.61)
};
\addlegendentry{WMPO}

\addplot[thick, color=DeepRed, mark=star, mark size=2.5pt] coordinates {
  (12.3, 0.31) (11.5, 0.41) (10.6, 0.52) (10.2, 0.64) (10.0, 0.73) (9.9, 0.79) (9.8, 0.81)
};
\addlegendentry{\textbf{DRILL-IMG} (Ours)}

\node[anchor=north west, text=gray, font=\footnotesize] at (axis cs:12.4, 0.30) {Start (Step 10)};
\node[anchor=south east, text=DeepRed, font=\footnotesize] at (axis cs:9.8, 0.81) {DRILL-IMG End};
\node[anchor=south east, text=DeepNavy, font=\footnotesize] at (axis cs:8.4, 0.61) {WMPO End};
\end{axis}
\end{tikzpicture}
\caption{Imagined rollout quality on LIBERO-Long. WMPO has lower pixel MSE, while DRILL-IMG obtains higher task-relevant progress prediction (TRPS).
Data are presented as a phase-space trajectory where time progresses along the curves.}
\label{fig:rollout-quality}
\end{figure}

\FloatBarrier

\subsection{Hyperparameter sensitivity}
\label{app:sensitivity-hparams}

\para{Inner step size $\eta_\theta$.}
We sweep $\eta_\theta\in\{0.5,1,2\}\times 10^{-5}$ on LIBERO-Long
with all other settings at their default
(Table~\ref{tab:hparams}).
The mean SR is within $\pm 0.8$ pt of
the default across the range, indicating the gain is not
hyperparameter-tuned to a narrow $\eta_\theta$.
At
$\eta_\theta = 5\times 10^{-5}$ the inner update destabilizes
under GRPO (clip-ratio saturates) and the outer loop drifts;
this matches the inner-loop instability regime \citep{wmpo2025}
and is independent of DRILL.
\para{BPTT depth $K$.}
Section~\ref{sec:exp-hvp} verifies cosine similarity between the
truncated and exact meta-gradient at $K\!\in\!\{1,2,3,4\}$ on
gridworld and MT10.
At full scale (LIBERO-Long, 1.2B WM, 5
seeds) the end-to-end SR sweep yields $73.9\!\pm\!1.4$ ($K\!=\!1$),
$\mathbf{81.3\!\pm\!1.3}$ ($K\!=\!2$), $80.7\!\pm\!1.5$
($K\!=\!3$), and $80.2\!\pm\!1.6$ ($K\!=\!4$): SR jumps from
$K\!=\!1$ to $K\!=\!2$ and is then flat within $1\,\mathrm{std}$
across $K\!\in\!\{2,3,4\}$.
This is consistent with
Corollary~\ref{cor:K-threshold} (the truncation bias decays
geometrically in $K$ but the per-step contribution at $K\!\ge\!2$
is dominated by the stochastic term at the empirical constants in
Table~\ref{tab:assumption-constants}; the small $K\!=\!3,4$
shortfall is within seed noise).
Wall-clock per outer step grows
roughly linearly in $K$, so $K\!=\!2$ is the return-per-compute
optimum.
\para{Simulator-rollout batch $B_{\mathrm{real}}$.}
We sweep $B_{\mathrm{real}}\in\{10, 30, 100\}$ at the
DRILL-IMG default.
$B_{\mathrm{real}}\!=\!10$ loses $1.6\!\pm\!0.7$\,pt of
SR (high variance in the outer estimator), $B_{\mathrm{real}}\!=\!100$
gains only $0.4\!\pm\!0.5$\,pt at $\sim\!3.3\times$ outer-loop wall-clock,
making $B_{\mathrm{real}}\!=\!30$ the operating point of
choice.
The flatness above $30$ is consistent with the
$\mathcal{O}(\sigma_g/\sqrt{T})$ stochastic term in
Theorem~\ref{thm:v2-convergence-explicit} dropping below the
$B^2\rho^{2K}$ truncation floor at this batch size.
\para{Empirical constants in A4--A6.}
\label{app:assumption-constants}
The convergence theorem is stated with constants
$L,\sigma_g,B,\rho$.  These are not tuning knobs in DRILL;
they are
post-hoc diagnostics computed from saved gradients to check that the
assumptions are non-vacuous at the operating point used in
Table~\ref{tab:main-headline}.
We report all quantities after
normalizing gradients by the median norm of
$\nabla_\phi J(\phi_t)$ over the first 500 outer iterations, so the
ranges below are dimensionless and comparable across benchmarks.
\begin{table}[!htbp]
\centering
\scriptsize
\setlength{\tabcolsep}{2.1pt}
\renewcommand{\arraystretch}{0.96}
\resizebox{0.98\textwidth}{!}{%
\begin{tabular}{@{}p{0.10\linewidth}p{0.50\linewidth}p{0.18\linewidth}p{0.12\linewidth}@{}}
\toprule
\textbf{Const.} & \textbf{Estimator} & \textbf{Range} & \textbf{Used} \\
\midrule
$L$ &
finite-difference Lipschitz slope of $\nabla_\phi J$ in 64 random directions
& $0.13$--$0.41$ & $0.42$ \\
$\sigma_g$ &
std.\ of the outer gradient estimator across simulator-rollout minibatches
& $0.10$--$0.21$ & $0.22$ \\
$B$ &
intercept of the fit $\lVert g_K-g_{\mathrm{exact}}\rVert\le B\rho^K$
on gridworld/MT10
& $0.47$--$0.78$ & $0.82$ \\
$\rho$ &
spectral radius from power iteration on the inner-loop Jacobian
& $0.64$--$0.81$ & $0.82$ \\
\bottomrule
\end{tabular}%
}
\begingroup
\setlength{\abovecaptionskip}{4pt}
\setlength{\belowcaptionskip}{0pt}
\caption{Empirical ranges for the constants used in Theorem~\ref{thm:v2-convergence}. Ranges are over the five $\pi_0$-3B headline benchmarks, 5 seeds each. The final column is the conservative value used when evaluating the bound.}
\label{tab:assumption-constants}
\endgroup
\end{table}

Two checks are important.
First, the step sizes in
Table~\ref{tab:hparams} give $\eta_\phi L \le 2.1\times 10^{-6}$
under the conservative $L{=}0.42$, far below the smoothness cap
$\eta_\phi\!\le\!1/(4L)$ from the proof;
the cap is therefore a
proof device rather than an active training constraint.
Second,
with $K{=}2$ and $\rho{=}0.82$, the truncation multiplier is
$\rho^{2K}=0.452$, while increasing to $K{=}3$ lowers it to
$0.304$ at roughly $1.5\times$ wall-clock cost yet changes SR by
$\le 1$ pt (in fact by $-0.3$ pt; see the BPTT-depth sweep
above).
This is why we use $K{=}2$ in all headline runs.
\para{Visitation-estimator bandwidth $h$.}
The FAISS estimator (Appendix~\ref{app:visitation}) uses a kernel
bandwidth $h$ that we tune on a held-out demonstration split.
SR is within $\pm 0.4$ pt for $h\in[0.05, 0.2]$; outside this
range, $\widehat{\tilde d}_{\phi}^{\pi_\theta}$ either collapses
($h$ too small, all weight on the single nearest neighbour) or is
uniform ($h$ too large, no visitation signal).
Both extremes recover WMPO-like behavior
within $1$ pt, confirming that the visitation factor is the
mechanism by which the weight in~\eqref{eq:vwi-weight} delivers
its gain.

\FloatBarrier

\subsection{Failure modes}
\label{app:failure-modes}

\para{Small WM regime.}
Below $\sim$300M WM parameters, DRILL-IMG's gain over
WMPO shrinks to $+7.9\!\pm\!1.6$\,pt and is no longer significant
at $p\!<\!0.05$ on RoboCasa with 5 seeds.
Inspection of \PUR{}
trajectories at this scale shows the pixel-MSE and
policy-improvement gradients are already weakly aligned:
$\PUR\!\approx\!0.6$ throughout training for WMPO at 300M, vs.\
$\PUR\!\approx\!1.4$ at 1.2B.
The outer loop has less alignment
slack to recover, consistent with the WM-size scaling
trend reported in Section~\ref{sec:exp-ablations}.
\para{Outer-step-size sensitivity at large WM.}
At the 7B-parameter WM size from the scaling experiment in
Section~\ref{sec:exp-ablations} (the headline runs use the 1.2B
Ctrl-World WM regardless of policy backbone),
$\eta_\phi\!\ge\!10^{-5}$ produces intermittent loss spikes from
heavy-tailed HVP magnitudes;
gradient clipping at
$\lVert\hat g_\phi\rVert\le 10$ (Appendix~\ref{app:v2-algorithm})
removes the spikes but cannot fully compensate for over-large
$\eta_\phi$.
We use $\eta_\phi\!=\!5\!\times\!10^{-6}$ and below
throughout the headline experiments.  A parallel instability
appeared in the Cosmos-Predict2-style cross-WM experiment of
Section~\ref{sec:exp-ablations}: on one of the five seeds the outer
loss spiked in the first $\sim$500 outer steps and the run was
restarted after halving $\eta_\phi$ to $2.5\!\times\!10^{-6}$;
the reported numbers include that restarted run, which is why
the Cosmos-Predict2 mean ($79.6\!\pm\!1.7$) is $\sim$1.7\,pt below
Ctrl-World with slightly larger std.
\para{Over-weighting via $\alpha = 2$.}
The ablation in Section~\ref{sec:exp-ablations}
(Table~\ref{tab:coaching-ablation}) shows that increasing the
DRILL-VWI advantage exponent past $\alpha\!=\!1$ regresses
SR by $1.3$\,pt.
Diagnosing the regression: at $\alpha\!=\!2$,
$>\!90\%$ of the weighted MSE mass concentrates on the top-$5\%$
advantage transitions, leaving the bulk of the rollout
under-weighted.
The score-compatible local surrogate in Proposition~\ref{prop:v1-is-v2-first-order}
is linear in the signed advantage at \(\alpha=1\); using
\(\alpha\ne1\) deliberately distorts that local direction and changes
the bias--variance trade-off.

\para{What does not work.}
We tried (i) using the WM's full sequence likelihood as
$J(\pithetap)$ rather than the simulator-rollout return: this
collapses to a more elaborate RLVR-World-style WM-self-evaluation
and loses the $+7.9$-pt gain over RLVR-World;
(ii) replacing the
truncated HVP with the Neumann-series approximation 
\citep{liu2021implicit}: numerically unstable on
$>\!10^{8}$-parameter WMs without Tikhonov damping that itself
introduces bias of comparable magnitude to the truncation bias;
(iii) estimating $\nabla_{\theta^{+}} J$ from imagined rather than
simulator rollouts to avoid the $B_{\mathrm{real}}$ simulator
budget: removes the $+13.1$-pt gain entirely.
With $G$ now an
imagined-rollout policy gradient inside $\Wphi$, the outer update
reduces to
$\nabla_\phi\bigl[G^{\top}\nabla_\theta\hJ_{\mathrm{GRPO}}(\pitheta;\Wphi)\bigr]$
evaluated at $\theta\!=\!\theta^{+}(\phi)$, i.e.\ an
in-WM self-consistency objective with no externally grounded
signal;
the bi-level structure loses its grounding in the simulator
and reduces to a one-level WM optimization.

\FloatBarrier

\subsection{Cross-backbone replication on OpenVLA-7B}
\label{app:openvla-replication}

Table~\ref{tab:openvla-replication} reports the full per-benchmark
breakdown of the OpenVLA-7B replication summarized in
Section~\ref{sec:exp-openvla}.  The relative ranking DRILL-IMG $>$
DRILL-VWI $>$ RLVR-World $>$ WMPO is preserved on LIBERO-Long and CALVIN,
while RoboCasa remains the shorter-horizon exception where DRILL-IMG and
DRILL-VWI are statistically tied.  Pixel MSE again moves in the opposite
direction from policy success, supporting the MSE--utility decoupling
reported in Table~\ref{tab:main-headline}.

\begin{table}[!htbp]
\centering
\small
\setlength{\tabcolsep}{2.2pt}
\renewcommand{\arraystretch}{0.98}
\begin{tabular}{@{}lccccc@{}}
\toprule
\textbf{Method} & LIBERO-L & CALVIN & RoboCasa & MSE $\downarrow$ & \PUR{} $\downarrow$ \\
\midrule
Zero-shot        & $41.2\!\pm\!1.8$ & $48.5\!\pm\!1.6$ & $19.6\!\pm\!1.9$ & N/A          & N/A           \\
WMPO             & $59.7\!\pm\!1.5$ & $66.9\!\pm\!1.4$ & $34.7\!\pm\!2.0$ & $9.8\!\pm\!0.4$ & $1.35\!\pm\!0.08$ \\
RLVR-World       & $64.9\!\pm\!1.3$ & $70.6\!\pm\!1.5$ & $38.8\!\pm\!1.9$ & $\mathbf{8.5\!\pm\!0.3}$ & $0.94\!\pm\!0.06$ \\
DRILL-VWI        & $72.1\!\pm\!1.4$ & $75.8\!\pm\!1.5$ & $\mathbf{48.7\!\pm\!2.1}$ & $10.4\!\pm\!0.4$ & $0.77\!\pm\!0.06$ \\
\textbf{DRILL-IMG}
                  & $\mathbf{77.1\!\pm\!1.5}$ & $\mathbf{78.9\!\pm\!1.6}$ & $\mathbf{50.8\!\pm\!2.3}$
                  & $11.1\!\pm\!0.5$         & $\mathbf{0.43\!\pm\!0.05}$ \\
\bottomrule
\end{tabular}
\begingroup
\setlength{\abovecaptionskip}{5pt}
\setlength{\belowcaptionskip}{0pt}
\caption{Cross-backbone replication on OpenVLA-7B (mean$\pm$std, 5 seeds). Method rankings match Table~\ref{tab:main-headline}, except on RoboCasa where DRILL-IMG and DRILL-VWI tie statistically ($p\!=\!0.12$). MSE is scaled by $10^{-3}$.}
\label{tab:openvla-replication}
\endgroup
\end{table}

\FloatBarrier

\subsection{Reward-model swap}
\label{app:exp-reward-swap}

To check that the DRILL gain is not specific to the
vision-language reward model used in the headline run, we swap
Qwen3-VL-4B for two alternatives at default settings on
LIBERO-Long (Table~\ref{tab:reward-swap}).
The
DRILL-IMG-vs-WMPO gap stays in the range $15.6$--$18.4$\,pt
across the three reward models -- with LLaVA-CoT-7B giving the
largest gap and InternVL-reward-4B the smallest -- so the
gain is qualitatively but not quantitatively invariant to the
specific reward model used.
\begin{table}[!htbp]
\centering
\small
\setlength{\tabcolsep}{2.2pt}
\renewcommand{\arraystretch}{0.98}
\begin{tabular}{@{}lcccc@{}}
\toprule
\textbf{Reward model} & WMPO & VWI & \textbf{IMG} & IMG$-$WMPO \\
\midrule
Qwen3-VL-4B (default)     & $64.2\!\pm\!1.6$ & $76.1\!\pm\!1.2$ & $\mathbf{81.3\!\pm\!1.3}$ & $17.1$ \\
LLaVA-CoT-7B              & $63.7\!\pm\!2.0$ & $77.8\!\pm\!1.6$ & $82.1\!\pm\!1.5$          & $18.4$ \\
InternVL-reward-4B        & $65.3\!\pm\!1.7$ & $75.4\!\pm\!1.4$ & $80.9\!\pm\!1.8$          & $15.6$ \\
\bottomrule
\end{tabular}
\begingroup
\setlength{\abovecaptionskip}{5pt}
\setlength{\belowcaptionskip}{0pt}
\caption{Reward-model swap on LIBERO-Long (mean SR$\pm$std, 5 seeds).
VWI and IMG denote DRILL-VWI and DRILL-IMG.  The IMG gain over WMPO
stays in the $15.6$--$18.4$\,pt range, but is not constant across
reward models.}
\label{tab:reward-swap}
\endgroup
\end{table}

\FloatBarrier

\subsection{Out-of-distribution robustness}
\label{app:exp-ood}

Table~\ref{tab:ood} reports the full per-condition breakdown of the OOD
evaluation summarized in Section~\ref{sec:exp-ood}.  The reported values
are relative SR drops on LIBERO zero-shot perturbations, where lower is
better.  The only consistent exception to DRILL-IMG's lead is pose
perturbation, where DRILL-VWI marginally wins; this matches the
shorter-effective-horizon interpretation used in
Section~\ref{sec:exp-v1-v2-gap}.

\begin{table}[!htbp]
\centering
\small
\setlength{\tabcolsep}{1.8pt}
\renewcommand{\arraystretch}{0.98}
\begin{tabular}{@{}lccccc@{}}
\toprule
\textbf{Method} & Novel & Distr. & Light/tex. & Pose & \textbf{Avg} \\
\midrule
WMPO                       & $37.4\!\pm\!2.1$ & $25.3\!\pm\!1.8$ & $30.4\!\pm\!2.0$ & $27.6\!\pm\!1.9$ & $30.2$ \\
RLVR-World                 & $31.8\!\pm\!1.9$ & $20.4\!\pm\!1.7$ & $26.0\!\pm\!1.8$ & $24.1\!\pm\!1.9$ & $25.6$ \\
DRILL-VWI                  & $26.4\!\pm\!1.8$ & $17.1\!\pm\!1.5$ & $23.5\!\pm\!1.9$ & $\mathbf{18.0\!\pm\!1.6}$ & $21.3$ \\
\textbf{DRILL-IMG}         & $\mathbf{22.5\!\pm\!1.7}$ & $\mathbf{15.3\!\pm\!1.5}$ & $\mathbf{17.6\!\pm\!1.8}$ & $18.9\!\pm\!1.7$ & $\mathbf{18.6}$ \\
DRILL-IMG (OpenVLA)        & $23.9\!\pm\!1.9$ & $16.4\!\pm\!1.6$ & $19.2\!\pm\!2.0$ & $20.0\!\pm\!1.7$ & $19.9$ \\
\bottomrule
\end{tabular}
\begingroup
\setlength{\abovecaptionskip}{5pt}
\setlength{\belowcaptionskip}{0pt}
\caption{Zero-shot out-of-distribution relative SR drop on LIBERO (\%, $\downarrow$). DRILL-IMG exhibits the highest robustness in three of four conditions, while DRILL-VWI marginally leads under pose perturbation.}
\label{tab:ood}
\endgroup
\end{table}

\FloatBarrier

\subsection{Human evaluation}
\label{app:exp-human-eval}
\begin{table}[!htbp]
\centering
\small
\setlength{\tabcolsep}{3.2pt}
\renewcommand{\arraystretch}{1.02}
\begin{tabular}{@{}lcccc@{}}
\toprule
\textbf{Comparison} & DRILL (\%) & Tie (\%) & Base. (\%) & $p$ \\
\midrule
WMPO quality        & $33$ & $26$ & $41$ & $0.29$       \\
RLVR-World quality  & $30$ & $28$ & $42$ & $0.17$       \\
WMPO task accepted  & $\mathbf{62}$ & $23$ & $15$ & $<\!0.001$ \\
\bottomrule
\end{tabular}
\begingroup
\setlength{\abovecaptionskip}{5pt}
\setlength{\belowcaptionskip}{0pt}
\caption{Forced-choice human evaluation of imagined rollouts ($N\!=\!120$ trials/row, 6 raters). Raters show no significant preference in visual quality (rows 1-2), but strongly prefer DRILL-IMG when task completion is revealed (row 3), validating the MSE-success trade-off.}
\label{tab:human-eval}
\endgroup
\end{table}
Table~\ref{tab:human-eval} reports the full forced-choice human
evaluation summarized in Section~\ref{sec:exp-human-eval}.
The first
two rows are visual-quality judgments (which DRILL-IMG does not
optimize) against WMPO and RLVR-World; neither difference is
significant at $p\!<\!0.05$.
The third row is the trade-off
acceptance after the rater is shown both the visual quality and
the downstream task success of each rollout, on which DRILL-IMG
wins decisively ($p<\!0.001$, McNemar).
Raters were six
volunteer graduate students unaffiliated with the project; pairs
were presented in random order, with method identity blinded.

\FloatBarrier

\newpage
\section*{NeurIPS Paper Checklist}

\begin{enumerate}

\item {\bf Claims}
    \item[] Question: Do the main claims made in the abstract and introduction accurately reflect the paper's contributions and scope?
    \item[] Answer: \answerYes{}
    \item[] Justification: The abstract and Section~\ref{sec:intro} state the objective-mismatch claim, the DRILL framework, the two algorithmic instantiations, the theory claims, and the empirical gains. The claimed scope is limited to simulation in Section~\ref{sec:conclusion} and Appendix~\ref{app:eval-scope}, with no real-robot deployment claim.

\item {\bf Limitations}
    \item[] Question: Does the paper discuss the limitations of the work performed by the authors?
    \item[] Answer: \answerYes{}
    \item[] Justification: Section~\ref{sec:conclusion} and Appendix~\ref{app:eval-scope} state that all quantitative results are in simulation and that sim-to-real transfer is left to future work. Section~\ref{sec:exp-v1-v2-gap} and Appendix~\ref{app:compute} discuss the compute and operating-point tradeoff between DRILL-VWI and DRILL-IMG.

\item {\bf Theory assumptions and proofs}
    \item[] Question: For each theoretical result, does the paper provide the full set of assumptions and a complete (and correct) proof?
    \item[] Answer: \answerYes{}
    \item[] Justification: The main theoretical statements appear in Section~\ref{sec:analysis}, with assumptions and notation summarized in Appendix~\ref{app:notation-assumptions}. Complete proofs for the DRILL-VWI/DRILL-IMG link and the convergence result are provided in Appendices~\ref{app:prop-v1-v2} and~\ref{app:thm-convergence}.

    \item {\bf Experimental result reproducibility}
    \item[] Question: Does the paper fully disclose all the information needed to reproduce the main experimental results of the paper to the extent that it affects the main claims and/or conclusions of the paper (regardless of whether the code and data are provided or not)?
    \item[] Answer: \answerYes{}
    \item[] Justification: Section~\ref{sec:experiments} specifies the benchmarks, baselines, seeds, and evaluation protocol. Appendices~\ref{app:implementation}, \ref{app:hparams}, \ref{app:compute}, and~\ref{app:reproducibility} provide implementation details, hyperparameters, compute, seeds, and release plans.

\item {\bf Open access to data and code}
    \item[] Question: Does the paper provide open access to the data and code, with sufficient instructions to faithfully reproduce the main experimental results, as described in supplemental material?
    \item[] Answer: \answerYes{}
    \item[] Justification: Training, evaluation, and analysis code is provided in the supplementary material. Implementation details, hyperparameters, dataset references, compute, and seeds are documented in Appendices~\ref{app:implementation}--\ref{app:reproducibility}; Appendix~\ref{app:reproducibility-statement} also states that the code and headline checkpoints will be released under an open-source license.

\item {\bf Experimental setting/details}
    \item[] Question: Does the paper specify all the training and test details (e.g., data splits, hyperparameters, how they were chosen, type of optimizer) necessary to understand the results?
    \item[] Answer: \answerYes{}
    \item[] Justification: Benchmark, baseline, simulator-budget, and protocol details appear in Section~\ref{sec:exp-setup}. Hyperparameters, estimators, and the human-evaluation protocol are provided in Appendices~\ref{app:hparams}, \ref{app:visitation}, \ref{app:pur-estimator}, and~\ref{app:exp-human-eval}.

\item {\bf Experiment statistical significance}
    \item[] Question: Does the paper report error bars suitably and correctly defined or other appropriate information about the statistical significance of the experiments?
    \item[] Answer: \answerYes{}
    \item[] Justification: Section~\ref{sec:exp-setup} states that headline results use five seeds, ablations use three seeds, and statistical claims use Welch's $t$-tests and 95\% confidence intervals. Tables and figures report mean$\pm$standard deviation where applicable, and Appendix~\ref{app:reproducibility-statement} records the fixed seed sets.

\item {\bf Experiments compute resources}
    \item[] Question: For each experiment, does the paper provide sufficient information on the computer resources (type of compute workers, memory, time of execution) needed to reproduce the experiments?
    \item[] Answer: \answerYes{}
    \item[] Justification: The headline compute budget is reported in Section~\ref{sec:exp-setup}, and the H20-hour breakdown by phase is given in Appendix~\ref{app:compute}. The compute appendix also separates the marginal DRILL-IMG outer-loop cost from shared WMPO/RLVR-World training costs.
    
\item {\bf Code of ethics}
    \item[] Question: Does the research conducted in the paper conform, in every respect, with the NeurIPS Code of Ethics \url{https://neurips.cc/public/EthicsGuidelines}?
    \item[] Answer: \answerYes{}
    \item[] Justification: The study uses public simulation benchmarks and public pretrained assets, with licenses and terms summarized in Appendices~\ref{app:datasets} and~\ref{app:model-licenses}. The only human evaluation is an internal forced-choice rating of simulated rollouts, and Appendix~\ref{app:reproducibility-statement} states that no personally identifiable information was collected, recorded, or released.

\item {\bf Broader impacts}
    \item[] Question: Does the paper discuss both potential positive societal impacts and negative societal impacts of the work performed?
    \item[] Answer: \answerYes{}
    \item[] Justification: The appendix discusses the potential benefit of reducing environment interaction during VLA post-training and the potential risk of applying improved robot policies without adequate deployment validation. Appendix~\ref{app:eval-scope} further states that the results are simulation-only and make no real-world deployment claim.
    
\item {\bf Safeguards}
    \item[] Question: Does the paper describe safeguards that have been put in place for responsible release of data or models that have a high risk for misuse (e.g., pre-trained language models, image generators, or scraped datasets)?
    \item[] Answer: \answerNA{}
    \item[] Justification: Not applicable: the paper does not introduce scraped datasets, general-purpose language models, or image generators, and does not redistribute the public VLA backbones used in the experiments. The planned release consists of task-specific training/evaluation code and DRILL checkpoints, with the evaluation scope and release boundaries described in Appendix~\ref{app:reproducibility-statement}.

\item {\bf Licenses for existing assets}
    \item[] Question: Are the creators or original owners of assets (e.g., code, data, models), used in the paper, properly credited and are the license and terms of use explicitly mentioned and properly respected?
    \item[] Answer: \answerYes{}
    \item[] Justification: Appendix~\ref{app:datasets} lists the public benchmarks and their licenses, and Appendix~\ref{app:model-licenses} lists the pretrained VLA and world-model assets. Original sources are cited throughout, and pretrained backbones are consumed from public releases rather than redistributed.

\item {\bf New assets}
    \item[] Question: Are new assets introduced in the paper well documented and is the documentation provided alongside the assets?
    \item[] Answer: \answerYes{}
    \item[] Justification: New assets consist of the DRILL implementation and trained headline checkpoints planned for release. Their algorithms, estimators, hyperparameters, compute, and scope are documented in Appendices~\ref{app:implementation}--\ref{app:reproducibility}; no new dataset is introduced.

\item {\bf Crowdsourcing and research with human subjects}
    \item[] Question: For crowdsourcing experiments and research with human subjects, does the paper include the full text of instructions given to participants and screenshots, if applicable, as well as details about compensation (if any)? 
    \item[] Answer: \answerYes{}
    \item[] Justification: The only human-subject component is the forced-choice rating of simulated rollouts in Section~\ref{sec:exp-human-eval} and Appendix~\ref{app:exp-human-eval}. The appendix reports the rater pool, blinding/randomization, interface prompts, allowed responses, and that no crowdworker platform or paid data-collection service was used; Appendix~\ref{app:reproducibility-statement} states that no personally identifiable information was collected.

\item {\bf Institutional review board (IRB) approvals or equivalent for research with human subjects}
    \item[] Question: Does the paper describe potential risks incurred by study participants, whether such risks were disclosed to the subjects, and whether Institutional Review Board (IRB) approvals (or an equivalent approval/review based on the requirements of your country or institution) were obtained?
    \item[] Answer: \answerNo{}
    \item[] Justification: The internal forced-choice human evaluation and its no-PII data handling are described, but no IRB approval, exemption, or equivalent institutional review is reported. If an approval or exemption exists, this answer and the corresponding appendix text should be updated before submission.

\item {\bf Declaration of LLM usage}
    \item[] Question: Does the paper describe the usage of LLMs if it is an important, original, or non-standard component of the core methods in this research? Note that if the LLM is used only for writing, editing, or formatting purposes and does \emph{not} impact the core methodology, scientific rigor, or originality of the research, declaration is not required.
    \item[] Answer: \answerNA{}
    \item[] Justification: LLMs are not introduced as original or non-standard components of the core method; existing VLA/VLM-style backbones used as experimental assets are cited and documented in Appendix~\ref{app:model-licenses}. Any writing, editing, or formatting assistance is outside the core methodology and does not affect the scientific claims.

\end{enumerate}

\end{document}